\RequirePackage{fix-cm}
\documentclass{svjour3}                     
\smartqed  
\usepackage{graphicx}

\usepackage{amsmath}
\usepackage{amssymb}
\usepackage{amsthm}
\usepackage{mathtools}
\usepackage[utf8]{inputenc}
\usepackage{enumerate}
\usepackage{csquotes}
\usepackage{float}
\usepackage{afterpage}
\usepackage{subcaption}
\usepackage{algorithm}
\usepackage{algorithmic}
\usepackage{tabularx}
\usepackage{mathptmx}
\usepackage{natbib}
\usepackage{bm}
\usepackage{hyperref}
\usepackage[labelfont=bf]{caption}

\setcitestyle{aysep={}} 
\newcommand{\boldx}{\bm{x}}
\newcommand{\boldX}{\bm{X}}
\newcommand{\boldxS}{\bm{x_{S}}}
\newcommand{\boldxMinusS}{\bm{x_{-S}}}
\newcommand{\boldxMinusj}{\bm{x_{-j}}}

\newcommand{\boldxi}{\bm{x^{(i)}}}

\newcommand{\boldxMinusji}{\bm{x_{-j}^{(i)}}}

\newcommand{\boldhS}{\bm{h}_{S}}

\usepackage{geometry}
\begin{document}

\title{Marginal Effects for Non-Linear Prediction Functions
\thanks{This work has been partially supported by the German Federal Ministry of Education and Research (BMBF) under Grant No. 01IS18036A. The authors of this work take full responsibilities
for its content. We thank the anonymous reviewers for their constructive comments, specifically on structuring the paper, on the line integral for the non-linearity measure, and on the instabilities of decision trees.}
}

\author{Christian A. Scholbeck \and
        Giuseppe Casalicchio  \and
        Christoph Molnar \and
        Bernd Bischl  \and
        Christian Heumann
}

\authorrunning{Scholbeck et al.} 

\institute{Christian A. Scholbeck \at
            Ludwig-Maximilians-Universität Munich, Department of Statistics, Munich, Germany \\
            \email{christian.scholbeck@stat.uni-muenchen.de}           
            \and
            Giuseppe Casalicchio \at
            Ludwig-Maximilians-Universität Munich, Department of Statistics, Munich, Germany
            \and 
            Christoph Molnar \at
            Ludwig-Maximilians-Universität Munich, Department of Statistics, Munich, Germany
            \and
            Bernd Bischl \at
            Ludwig-Maximilians-Universität Munich, Department of Statistics, Munich, Germany
             \and
            Christian Heumann \at
            Ludwig-Maximilians-Universität Munich, Department of Statistics, Munich, Germany
}

\def\makeheadbox{\relax}
\setlength{\textwidth}{\dimexpr\pdfpagewidth-2in}
\date{}

\maketitle

\begin{abstract}
Beta coefficients for linear regression models represent the ideal form of an interpretable feature effect. However, for non-linear models and especially generalized linear models, the estimated coefficients cannot be interpreted as a direct feature effect on the predicted outcome. Hence, marginal effects are typically used as approximations for feature effects, either in the shape of derivatives of the prediction function or forward differences in prediction due to a change in a feature value. While marginal effects are commonly used in many scientific fields, they have not yet been adopted as a model-agnostic interpretation method for machine learning models. This may stem from their inflexibility as a univariate feature effect and their inability to deal with the non-linearities found in black box models. We introduce a new class of marginal effects termed forward marginal effects. We argue to abandon derivatives in favor of better-interpretable forward differences. Furthermore, we generalize marginal effects based on forward differences to multivariate changes in feature values. To account for the non-linearity of prediction functions, we introduce a non-linearity measure for marginal effects.
We argue against summarizing feature effects of a non-linear prediction function in a single metric such as the average marginal effect. Instead, we propose to partition the feature space to compute conditional average marginal effects on feature subspaces, which serve as conditional feature effect estimates.
\end{abstract}

\section{Introduction}
\label{section:introduction}

The lack of interpretability of most machine learning (ML) models has been considered one of their major drawbacks \citep{breiman_twocultures}. As a consequence, researchers have developed a variety of model-agnostic techniques to explain the behavior of ML models. These techniques are commonly referred to by the umbrella terms of interpretable machine learning (IML) or explainable artificial intelligence. In this paper, we focus on feature effects that indicate the direction and magnitude of a change in a predicted outcome due to a change in feature values. We distinguish between local explanations on the observational level and global ones for the entire feature space. In many applications, we are concerned with feature effects; for example, in medical research, we might want to assess the increase in risk of contracting a disease due to a change in a patient's health characteristics such as age or body weight.
\par
Consider the interpretation of a linear regression model (LM) without interaction terms where $\beta_j$ denotes the coefficient of the $j$-th feature. Increasing a feature value $x_j$ by one unit causes a change in predicted outcome of $\beta_j$. LMs are therefore often interpreted by merely inspecting the estimated coefficients. When the terms are non-linear, interactions are present, or when the expected target is transformed such as in generalized linear models (GLMs), interpretations are both inconvenient and unintuitive. For instance, in logistic regression, the expectation of the target variable is logit-transformed, and the predictor term cannot be interpreted as a direct feature effect on the predicted risk. It follows that even linear terms have a non-linear effect on the predicted target that varies across the feature space and makes interpretations through the model parameters difficult to impossible. A more convenient and intuitive interpretation corresponds to the derivative of the prediction function w.r.t. the feature or inspecting the change in prediction due to an intervention in the data. These two approaches are commonly referred to as marginal effects (MEs) in statistical literature \citep{bartus_marginal_effects}. MEs are often aggregated to an average marginal effect (AME), which represents an estimate of the expected ME. Furthermore, marginal effects at means (MEM) and marginal effects at representative values (MER) correspond to MEs where all features are set to the sample mean or where some feature values are set to manually chosen values \citep{williams_margins}.
These can be used to answer common research questions, e.g., what the average effect of age or body weight is on the risk of contracting the disease (AME), what the effect is for a patient with average age and body weight (MEM), and what the effect is for a patient with pre-specified age and body weight values (MER).
An increasing amount of scientific disciplines now rely on the predictive power of black box ML models instead of using intrinsically interpretable models such as GLMs, e.g., econometrics \citep{athey_policy} or psychology \citep{stachl_smartphone}. This creates an incentive to review and refine the theory of MEs for the application to non-linear models.

\subsection{Motivation} 

In their current form, MEs are not an ideal tool to interpret many statistical models such as GLMs. Furthermore, the shortcomings of MEs are exacerbated when applying MEs to black box models such as the ones created by many ML algorithms.
The ME is an unfit tool to determine multivariate feature effects, which is a requirement for model interpretability. Moreover, for non-linear prediction functions, MEs based on derivatives provide misleading interpretations, as the derivative is evaluated at a different location in the feature space where it may substantially deviate from the prediction function. The alternative and often overlooked definition based on forward differences suffers from a loss in information about the shape of the prediction function (see Section \ref{sec:recommendations_shortcomings}).
Furthermore, predictive models typically do not behave reliably in areas with a low density of training observations. MEs are often based on such model extrapolations in order to answer specific research questions \citep{hainmueller_trust}. For instance, one may evaluate feature effects on the disease risk for an artificially created patient observation which was not included in the training data. This problem is aggravated when switching from derivatives to forward differences (see Section \ref{sec:recommendations_shortcomings}).
Lastly, for linear models, the ME is identical across the entire feature space. For non-linear models, one typically estimates the global feature effect by computing the AME \citep{bartus_marginal_effects, onukwugha_marginaleffects}. However, a global average does not accurately represent the nuances of a non-linear predictive model. A more informative summary of the prediction function corresponds to the conditional feature effect on a feature subspace, e.g., patients with an entire range of health characteristics might be associated with homogeneous feature effects. Instead of global interpretations on the entire feature space, one should instead aim for semi-aggregated (semi-global) interpretations. More specifically, one should work towards computing multiple, semi-global conditional AMEs (cAMEs) instead of a single, global AME.

\subsection{Contributions}

The purpose of this research paper is twofold. First, we aim to provide a review of the theory of MEs to compute feature effects. Second, we refine MEs for the application to non-linear models. To avoid misinterpretations resulting from applying derivative MEs (dMEs) to non-linear prediction functions, we argue to abandon them in favor of forward differences which we term forward marginal effects (fMEs). Furthermore, we extend the univariate definition of fMEs to multivariate feature changes in order to determine multivariate feature effects. We demonstrate the superiority of fMEs over dMEs to interpret tree-based prediction functions. Although the resulting fMEs provide a precise feature effect measure on the predicted outcome, we are at risk of misjudging the shape of the prediction function as a linear function. To counteract this loss in information, we propose a non-linearity measure (NLM) for fMEs based on the similarity between the prediction function and the intersecting linear secant. Furthermore, for a more nuanced interpretation, we introduce conditional MEs (cMEs) on feature subspaces as a semi-global feature effect measure that more accurately describes feature effects across the feature space. Moreover, we propose to estimate cMEs with conditional AMEs (cAMEs), which can be computed by recursively partitioning the feature space with a regression tree on fMEs. Furthermore, we provide proofs on additive recovery for the univariate and multivariate fME and a proof on the relation between the individual conditional expectation (ICE) / partial dependence (PD) and the fME / forward AME.

\subsection{Structure of the Paper}

The paper is structured as follows: In Section \ref{sec:notation_background}, we introduce the notation and general background on predictive modeling and finite differences.
In Section \ref{sec:review}, we review existing methodology of MEs. We begin with a survey of related work in theoretical and applied statistics, IML, and sensitivity analysis. We review existing definitions of MEs, including categorical MEs, (numeric) derivatives and univariate forward differences, as well as variants and aggregations of MEs, i.e., the AME, MEM and MER. In Section \ref{sec:recommendations_shortcomings}, we provide recommendations for current practice and discuss shortcomings of existing methodology. We argue to abandon dMEs in favor of fMEs and discuss the selection of step sizes, including potential model extrapolations. In Section \ref{sec:novel_methodology}, we introduce a new class of fMEs, i.e., an observation-wise definition of the categorical ME, the multivariate fME, the NLM, the cME, and a framework to estimate the cME via cAMEs. In Section \ref{sec:comparison_lime}, we compare fMEs -- including our introduced novelties -- to the competing state-of-the-art method LIME. In Section \ref{sec:simulations}, we run multiple simulations with fMEs and the NLM. In Section \ref{sec:workflow_application}, we present our suggested application workflow and an applied example with the Boston housing data. The Appendix contains background information on additive decompositions of prediction functions and on model extrapolations, as well as several mathematical proofs.



\section{Notation and General Background}
\label{sec:notation_background}

This section introduces the notation and general background information on predictive modeling and finite differences. See Appendix \ref{app:background} for further background information on additive decompositions of prediction functions and on model extrapolations.

\subsection{Data and Predictive Model}
We consider a $p$-dimensional feature space $\mathcal{X} = \mathcal{X}_1 \times \dots \times \mathcal{X}_p$ and a target space $\mathcal{Y}$. The random variables on the feature space are denoted by $\boldX = (X_1, \dots, X_p)$\footnote{Vectors are denoted in bold letters.}. The random variable on the target space is denoted by $Y$. An undefined subspace of all features is denoted by $\mathcal{X}_{[\;]} \subseteq \mathcal{X}$. Correspondingly, $\boldX$ with a restricted sample space is denoted by $\boldX_{[\;]}$.
A realization of $\boldX$ and $Y$ is denoted by $\boldx = (x_1, \dots, x_p)$ and $y$. The probability distribution $\mathcal{P}$ is defined on the sample space $\mathcal{X} \times \mathcal{Y}$. A learning algorithm trains a predictive model $\widehat{f}: \mathbb{R}^p \mapsto \mathbb{R}$ on data drawn from $\mathcal{P}$, where $\widehat{f}(\boldx)$ denotes the model prediction based on the $p$-dimensional feature vector $\boldx$. To simplify our notation, we only consider one-dimensional predictions. However, the results on MEs can be generalized to multi-dimensional predictions, e.g., for multi-class classification.
We denote the value of the $j$-th feature in $\boldx$ by $x_j$. A set of features is denoted by $S \subseteq \{1, \dots, p\}$. The values of the feature set are denoted by $\boldxS$\footnote{As $\boldxS$ is the generalization of $x_j$ to vectors, we denote it in bold letters. However, it can in fact be a scalar. The same holds for $\boldxMinusS$ and $\boldxMinusj$.}. All complementary features are indexed by $-j$ or $-S$, so that $\boldxMinusj = \bm{x_{\{1, \, \dots \, , \, p\} \; \setminus \; \{j\}}}$, or $\mathbf{\boldxMinusS} = \bm{x_{\{1, \, \dots\, , \, p\} \; \setminus \; S}}$. An instance $\boldx$ can be partitioned so that $\boldx = (x_j, \boldxMinusj)$, or $\boldx = (\boldxS, \boldxMinusS)$. With slight abuse of notation, we may denote the vector $\boldxS$ by $(x_1, \dots, x_s)$ regardless of the elements of $S$, or the vector $(x_j, \boldxMinusj)$ by $(x_1, \dots, x_j, \dots, x_p)$ although $j \in \{1, \dots, p\}$.
The $i$-th observed feature vector is denoted by $\boldxi$ and corresponds to the target value $y^{(i)}$. 
We evaluate the prediction function with a set of training or test data $\mathcal{D} = \left\{\boldxi\right\}_{i = 1}^n$.

\subsection{Finite Differences and Derivatives}
A finite difference (FD) of the prediction function $\widehat{f}(\boldx)$ w.r.t. $x_j$ is defined as:
\begin{equation*}
FD_{j, \boldx, a, b} = \widehat{f}(x_1, \dots, x_j + a, \dots, x_p) - \widehat{f}(x_1, \dots, x_j + b, \dots, x_p)
\end{equation*}
The FD can be considered a movement on the prediction function (see Fig. \ref{fig:FD_movement}). There are three common variants of FDs: forward ($a = h$, $b = 0$), backward ($a = 0$, $b = -h$), and central differences ($a = h$, $b = -h$). In the following, we only consider forward differences with $b = 0$ where the FD is denoted without $b$. Dividing the FD by $(a-b)$ corresponds to the difference quotient:
\begin{equation}
\frac{FD_{j, \boldx, a, b}}{a - b} =\frac{\widehat{f}(x_1, \dots, x_j + a, \dots, x_p) - \widehat{f}(x_1, \dots, x_j + b, \dots, x_p)}{a - b}
\end{equation}
\begin{figure}
\centering
 \includegraphics[width=0.5\textwidth]{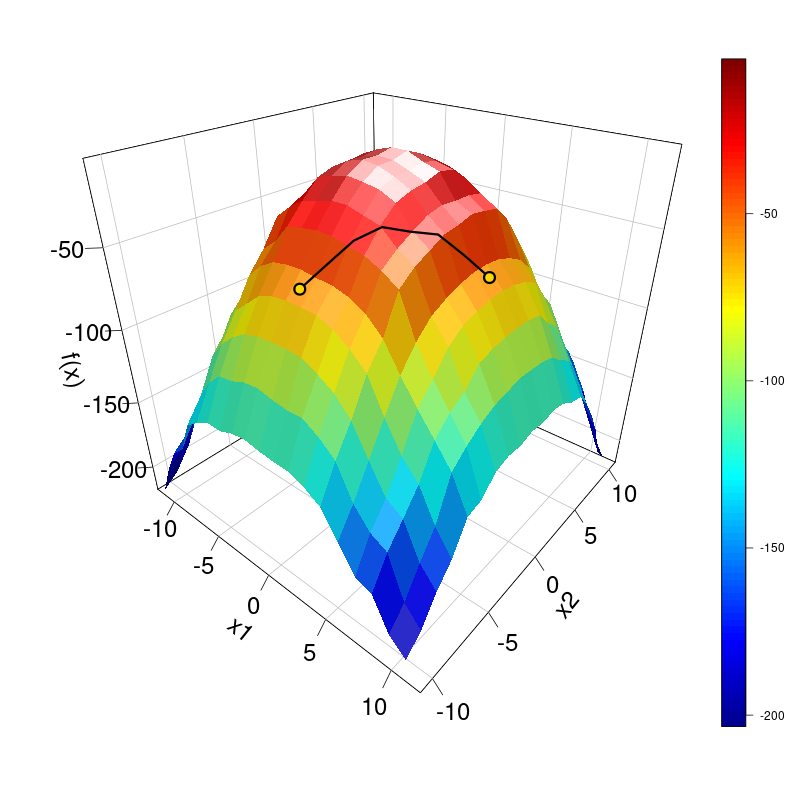}
\caption{\label{fig:FD_movement} The surface represents an exemplary prediction function dependent on two features. The FD can be considered a movement on the prediction function. We travel from point (0, -7) to point (7, 0).}
\end{figure}
\par
The derivative is defined as the limit of the forward difference quotient when $a = h$ approaches zero:
\begin{equation*}
\frac{\partial \widehat{f}(\boldX)}{\partial X_j}_{\big \vert \; \boldX = \boldx} = \lim\limits_{h \rightarrow 0} \frac{\widehat{f}(x_1, \dots, x_j + h, \dots, x_p) - \widehat{f}(\boldx)}{h}
\end{equation*}
We can numerically approximate the derivative with small values of $h$. We can use forward, backward, or symmetric FD quotients, which have varying error characteristics. As an example, consider a central FD quotient which is often used for dMEs \citep{leeper_margins}:
\begin{gather*}
\frac{\partial \widehat{f}(\boldX)}{\partial X_j}_{\big \vert \; \boldX = \boldx} \approx \frac{\widehat{f}(x_1, \dots, x_j + h, \dots, x_p) - \widehat{f}(x_1, \dots, x_j - h, \dots, x_p)}{2h} \quad , \quad h > 0
\end{gather*}

\section{Review of Existing Methodology}
\label{sec:review}
This section provides a review of existing methodology. We survey related work and review existing definitions of MEs, including categorical MEs, (numeric) derivatives and univariate forward differences, and variants and aggregations of MEs.

\subsection{Related Work}

\subsubsection{Statistics and Applied Fields}
MEs have been discussed extensively in the literature on statistics and statistical software, e.g., by \citet{ai_marginal_effects}, \citet{greene_econometrics}, \citet{norton_marginal_effects}, or \citet{mullahy_marginal_effects}.
The \texttt{margins} command is a part of Stata \citep{stata_manual} and was originally implemented by \citet{bartus_marginal_effects}. A brief description of the \texttt{margins} command is given by \citet{williams_margins}.
\citet{leeper_margins} provides an overview on dMEs and their variations as well as a port of \texttt{Stata's} functionality to \texttt{R}.
\citet{ramsey_me_neural_network} compute an ME for a single-hidden layer forward back-propagation artificial neural network by demonstrating its equivalent interpretation with a logistic regression model with a flexible index function. \citet{zhao_me_travel_application} apply model-agnostic dMEs to ML models in the context of analyzing travel behavior. Furthermore, they mention the unsuitability of derivatives for tree-based prediction functions such as random forests.
\par
\citet{mize_discrete_change} provide a test framework for cross-model differences of MEs. They refer to a marginal effect with a forward difference as a discrete change and to the corresponding AMEs as average discrete changes. 
\citet{gelman_predictive_comparisons} propose the predictive effect as a local feature effect measure. The predictive effect is a univariate forward difference, divided by the change in feature values (i.e., the step size). This differentiates it from the fME which we also define for multivariate feature changes and which is not divided by the step size, i.e., it provides a change in prediction as opposed to a rate of change. Furthermore, the authors propose an average predictive effect that corresponds to the average of multiple predictive effects that were measured at distinct feature values and model parameters. As such, it is a generalization of the AME that may be estimated with artificially created data points (as opposed to the sample at hand) and incorporates model comparisons (measured with different model parameters).

\subsubsection{Interpretable Machine Learning}

The most commonly used techniques to determine feature effects include the ICE \citep{goldstein_ice}, the PD \citep {friedman_pdp}, accumulated local effects (ALE) \citep{apley_ale}, Shapley values \citep{strumbelj_shapley}, Shapley additive explanations (SHAP) \citep{lundberg_shap}, local interpretable model-agnostic explanations (LIME) \citep{ribeiro_lime}, and counterfactual explanations \citep{wachter_counterfactuals}. Counterfactual explanations indicate the smallest necessary change in feature values to receive the desired prediction and represent the counterpart to MEs. \citet{goldstein_ice} propose derivative ICE (d-ICE) plots to detect interactions. The d-ICE is a univariate ICE where the numeric derivative w.r.t. the feature of interest is computed pointwise after a smoothing procedure. Symbolic derivatives are commonly used to determine the importance of features for neural networks \citep{ancona_gradients}.
While MEs provide interpretations in terms of prediction \textit{changes}, most methods provide an interpretation in terms of prediction \textit{levels}. LIME is an alternative option that returns interpretable parameters (i.e., rates of change in prediction) of a local surrogate model. LIME, and to a lesser extent SHAP, have been demonstrated to provide unreliable interpretations in some cases. For instance, LIME is strongly influenced by the chosen kernel width parameter \citep{slack_fooling_lime}. In Section \ref{sec:comparison_lime}, we compare our new class of fMEs to LIME.
Furthermore, many techniques in IML are interpreted visually (e.g., ICEs, the PD, ALE plots) and are therefore limited to feature value changes in at most two dimensions. MEs are not limited by the dimension of the intervention in feature values, as any change in feature values -- regardless of its dimensionality -- always results in a single ME.

\subsubsection{Sensitivity Analysis}

The goal of sensitivity analysis (SA) is to determine how uncertainty in the model output can be attributed to uncertainty in the model input, i.e., determining the importance of input variables \citep{saltelli_sa}. Techniques based on FDs are common in SA \citep{razavi_sa_review}.
The numeric derivative of the function to be evaluated w.r.t. an input variable serves as the natural definition of local importance in SA. The elementary effect (EE) was first introduced as part of the Morris method \citep{morris_method} as a screening tool for important inputs. The EE corresponds to a univariate forward difference quotient with variable step sizes, i.e., it is a generalization of the derivative. Variogram-based methods analyze forward differences computed at numerous pairs of points across the feature space \citep{razavi_variogram}.
Derivative-based global sensitivity measures (DGSM) provide a global feature importance metric by averaging derivatives at points obtained via random or quasi-random sampling.

\subsection{Background on Marginal Effects}

\subsubsection{Categorical Features}

MEs for categorical features are often computed as the change in prediction when the feature value changes from a reference category to another category \citep{williams_margins}. In other words, for each observation, the observed categorical feature value is set to the reference category, and we record the change in prediction when changing it to every other category. Given $k$ categories, this results in $k-1$ MEs for each observation. Consider a categorical feature $x_j$ with categories $C = \{c_1, \dots, c_k\}$. We denote the reference category by $c_r$. The categorical ME for an observation $\boldx$ and a single category $c_l$ corresponds to Eq. (\ref{eq:categorical_me}):
\begin{align}
\text{ME}_{j, \boldx, c_r, c_l} = \widehat{f}(c_l, \boldxMinusj) - \widehat{f}(c_r, \boldxMinusj) 
\label{eq:categorical_me}
\end{align}

\subsubsection{Derivatives for Continuous Features}

The most commonly used definition of MEs for continuous features corresponds to the derivative of the prediction function w.r.t. a feature. We will refer to this definition as the derivative ME (dME). In case of a linear prediction function, the interpretation of dMEs is simple: if the feature value increases by one unit, the prediction will increase by the dME estimate. Note that even the prediction function of a linear regression model can be non-linear if exponents of order $\geq 2$ are included in the feature term. Similarly, in GLMs, the linear predictor is transformed (e.g., log-transformed in Poisson regression or logit-transformed in logistic regression). Interpreting the dME is problematic for non-linear prediction functions, which we further explore in the following section.

\subsubsection{Forward Differences with Univariate Feature Value Changes}

A distinct and often overlooked definition of MEs corresponds to the change in prediction with adjusted feature values. Instead of computing the tangent at the point of interest and inspecting it at another feature value, this definition corresponds to putting the secant through the point of interest and the prediction with another feature value (see Fig. \ref{fig:ME_secant}). This definition of MEs is based on a forward difference instead of a symmetric difference. We refer to this variant as the forward ME (fME). Unlike the dME, this variant does not require dividing the FD by the interval width. Typically, one uses a one-unit-change ($h_j = 1$) in feature values \citep{mize_discrete_change}. 
\begin{align}
\text{fME}_{\boldx, h_j} &= \widehat{f}(x_j + h_j, \boldxMinusj) - \widehat{f}(\boldx)
\label{eq:univariate_fME}
\end{align}
The fME is illustrated in Fig. \ref{fig:ME_secant}. It corresponds to the change in prediction along the secant (green) through the point of interest (prediction at x = 0.5) and the prediction at the feature value we receive after the feature change (x = 1.5).
\begingroup
\renewcommand*{\arraystretch}{1.25}
\begin{table}
\caption{\label{tab:}\label{tab:boston_housing_observation}Observation used to compute the fME of the feature \texttt{rm} with a step size of 1, before (top) and after (bottom) the intervention in the Boston housing data.}
\centering
\resizebox{\linewidth}{!}{
\begin{tabular}[t]{|r|r|r|r|r|r|r|r|r|r|r|r|r|r|}
\hline
crim & zn & indus & chas & nox & \textbf{rm} & age & dis & rad & tax & ptratio & afram & lstat & medv\\
\hline
0.02498 & 0 & 1.89 & 0 & 0.518 & \textbf{6.54} & 59.7 & 6.2669 & 1 & 422 & 15.9 & 389.96 & 8.65 & 16.5\\
\hline
0.02498 & 0 & 1.89 & 0 & 0.518 & \textbf{7.54} & 59.7 & 6.2669 & 1 & 422 & 15.9 & 389.96 & 8.65 & 16.5\\
\hline
\end{tabular}}
\end{table}
\endgroup

Table \ref{tab:boston_housing_observation} gives an exemplary intervention in the Boston housing data to compute an fME. The goal is to determine the fME of increasing the average number of rooms per dwelling by one room on the median value of owner-occupied homes per Boston census tract.

\subsubsection{Variants and Aggregations of Marginal Effects}

There are three established variants or aggregations of univariate MEs: The AME, MEM, and MER \citep{williams_margins}, which can be computed for both dMEs and fMEs. In the following, we will use the notation of fMEs. Although we technically refer to the fAME, fMEM and fMER, we omit the \enquote{forward} prefix in this case for reasons of simplicity.
\begin{enumerate}[(i)]
\item Average marginal effects (AME): The AME represents an estimate of the expected ME w.r.t. the distribution of $\boldX$. We estimate it via Monte Carlo integration, i.e., we average the fMEs that were computed for each (randomly sampled) observation:
\begin{align*}
 \mathbb{E}_{\boldX}\left[\text{fME}_{\boldX, h_j} \right] &= \mathbb{E}_{\boldX}\left[\widehat{f}(X_j + h_j, \bm{X}_{-j}) - \widehat{f}(\boldX)\right] \\
\text{AME}_{\mathcal{D}, h_j} &= \frac{1}{n} \sum_{i = 1}^n \left[\widehat{f}\left(x_j^{(i)} + h_j, \boldxMinusji\right) - \widehat{f}\left(\boldxi\right) \right]
\end{align*}
\item Marginal effect at means (MEM): The MEM can be considered the reverse of the AME, i.e., it is the ME evaluated at the expectation of $\boldX$. We estimate the MEM by substituting all feature values with their sampling distribution means:
\begin{align*}
\text{fME}_{\left(\mathbb{E}_{X_j}(X_j) + h_j, \mathbb{E}_{\bm{X}_{-j}}(\bm{X}_{-j})\right), \;h_j} &= \widehat{f}\left(\mathbb{E}_{X_j}(X_j) + h_j, \mathbb{E}_{\bm{X}_{-j}}(\bm{X}_{-j})\right) - \widehat{f}\left(\mathbb{E}_{\boldX}(\boldX)\right) \\
\text{MEM}_{\mathcal{D}, h_j} &= \widehat{f}\left(\left(\frac{1}{n} \sum_{i = 1}^n x_j^{(i)}\right) + h_j, \frac{1}{n} \sum_{i = 1}^n \boldxMinusji \right) - \widehat{f}\left(\frac{1}{n} \sum_{i = 1}^n \boldxi \right)
\end{align*}
Note that averaging values is only sensible for continuous features. \citet{williams_margins} defines a categorical MEM where all remaining features $\boldxMinusj$ are set to their sample means (conditional on being continuous) and $x_j$ changes from each category to a reference category.
\item Marginal effects at representative values (MER): Furthermore, we can substitute specific feature values for all observations by manually specified values $\boldx^{*}$. It follows that the MEM is a special case of the MER where the specified values correspond to the sample means. MERs can be considered conditional MEs, i.e., we compute MEs while conditioning on certain feature values. The MER for a single observation with modified feature values $\boldx^{*}$ corresponds to:
\begin{align*}
 \text{MER}_{\boldx^{*}, h_j} &= \widehat{f}\left(x^{*}_j + h_j, \boldxMinusj^{*}\right) - \widehat{f}\left(\boldx^{*}\right)
\end{align*}
\end{enumerate}

\section{Recommendations for Current Practice and Shortcomings}
\label{sec:recommendations_shortcomings}

This section gives recommendations for current practice and discusses shortcomings of existing methods that form the basis for the modifications we present in the subsequent section.

\subsection{Forward Difference versus Derivative}

The most common way of using MEs in practice is through dMEs. In case of non-linear prediction functions, dMEs can lead to substantial misinterpretations (see Fig. \ref{fig:ME_secant}). The slope of the tangent (grey) at the point of interest (prediction at x = 0.5) corresponds to the dME. The default way to interpret the dME is to evaluate the tangent at the feature value we receive after a unit change (x = 1.5). This leads to substantial misinterpretations for non-linear prediction functions. In this case, there is an error (red) almost as large as the actual change in prediction. Although the computation of the dME does not require a step size, its interpretation does and is therefore error-prone. We recommend to use fMEs instead of the more popular dMEs. An fME always indicates an exact change in prediction for any prediction function and is therefore much more interpretable. For linear prediction functions, the interpretation of both variants is equivalent. MEs are often discarded in favor of differences in adjusted predictions (APs), i.e., predictions computed with different feature values \citep{williams_margins}. In the following section, we formulate a multivariate fME which bridges the gap between fMEs and differences in APs.

\subsection{Step Size and Extrapolation Risk}

The step size is determined both by the question that is being addressed and the scale of the feature \textit{at training time}. Consider a feature \texttt{weight} that is scaled in kilograms. If we want to evaluate the change in prediction if each observation's bodyweight increases by a kilogram, the appropriate step size is $1$. We could also assess the change in prediction if each observation's weight decreases by a kilogram by specifying a step size of $-1$. However, if \texttt{weight} is scaled in grams, the appropriate step sizes are $1000$ and $-1000$, respectively. Furthermore, one could think of dispersion-based measures to select step sizes. \citet{mize_discrete_change} suggest one standard deviation in observed feature values as an informative step size. Other options include, e.g., a percentage of the interquartile range or the mean / median absolute deviation.
\par
The step size cannot vary indefinitely without risking model extrapolations. This issue becomes especially prevalent when switching from dMEs (with instantaneous changes in feature values) to fMEs (with variable step sizes of arbitrary magnitude). Furthermore, when using non-training data (or training data in low-density regions) to compute fMEs, we are at risk of the model extrapolating without actually traversing the feature space. 
If the points $\boldx$ or $(x_j + h_j, \boldxMinusj)$ are classified as extrapolation points (EPs), the fME should be interpreted with caution. If the point $\boldx$ is located inside an extrapolation area (e.g., if it is located outside of the multivariate envelope of the training data), we suggest to exclude $\boldx$ from the data used to compute fMEs. Given the point $\boldx$ is not an EP, but $(x_j + h_j, \boldxMinusj)$ is, the step size can be modified until $(x_j + h_j, \boldxMinusj)$ is located inside a non-extrapolation area. It is debatable how to classify points as EPs. In Appendix \ref{app:extrapolation}, we describe several options for extrapolation detection. In Section \ref{sec:novel_methodology}, we extend the univariate definition of fMEs to multivariate steps. The same holds for multivariate fMEs, i.e., one can modify the marginal step widths independently until the model does not extrapolate anymore.
\par
Fig. \ref{fig:uniform_extrapolation} demonstrates the pitfall of model extrapolations when using fMEs. We draw points of a single feature $x$ from a uniform distribution on the interval $[-5, 5]$ with $y = x^2 + \epsilon$, where $\epsilon \sim N(0, 1)$. A random forest is trained to predict $y$. All points $x \not \in [-5, 5]$ fall outside the range of the training data and are classified as EPs. We compute fMEs with a step size of 1. By implication, all fMEs with $x > 4$ are based on model extrapolations. The resulting fMEs in the extrapolation area behave considerably different from the fMEs in the non-extrapolation area and should not be used for interpretation purposes, as they convey an incorrect impression of the feature effect of $x$ on $y$.
\begin{figure}
\centering
 \includegraphics[width=0.49\linewidth]{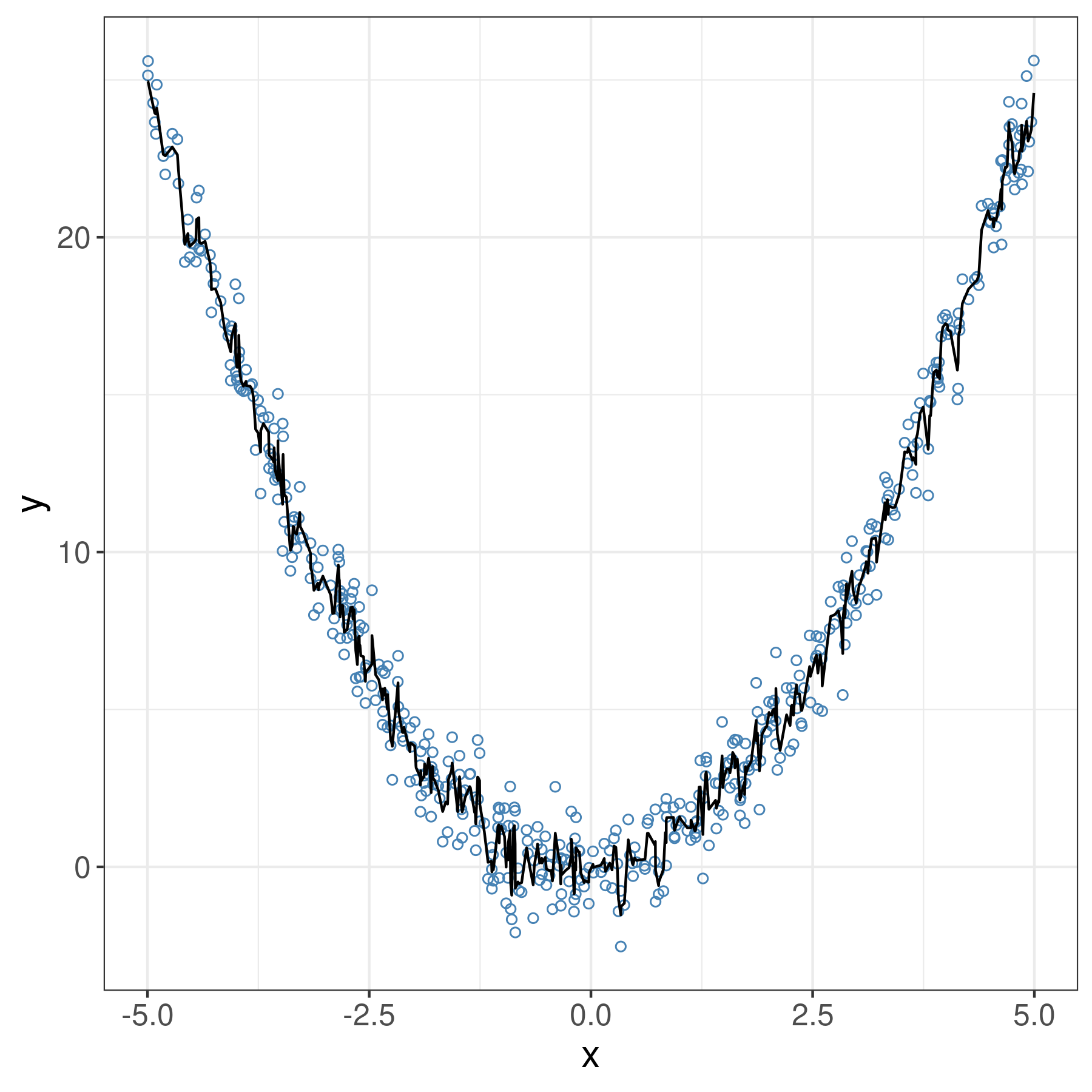}
 \includegraphics[width=0.49\linewidth]{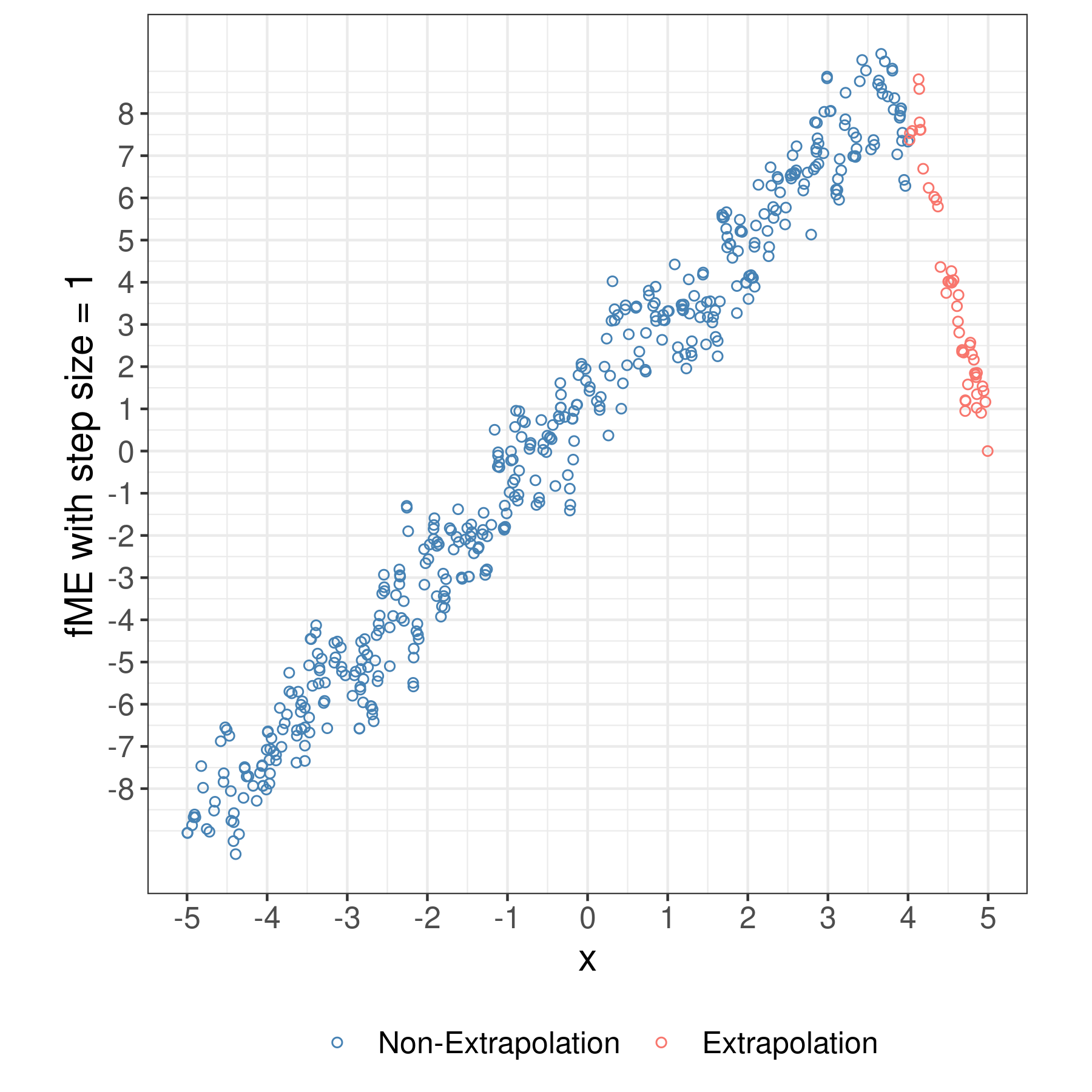}
\caption{\label{fig:uniform_extrapolation}
\textbf{Left}: A random forest is trained on a single feature $x$ with a quadratic effect on the target. The training space corresponds to the interval $[-5, 5]$. \textbf{Right}: We compute an fME with a step size of 1 for each observation. After moving 1 unit in $x$ direction, points with $x > 4$ are considered EPs (red). The random forest extrapolates and predicts unreliably in this area of the feature space. The resulting fMEs are irregular and should not be used for interpretation purposes.}
\end{figure}

\begin{figure}
\centering
 \includegraphics[width=0.5\textwidth]{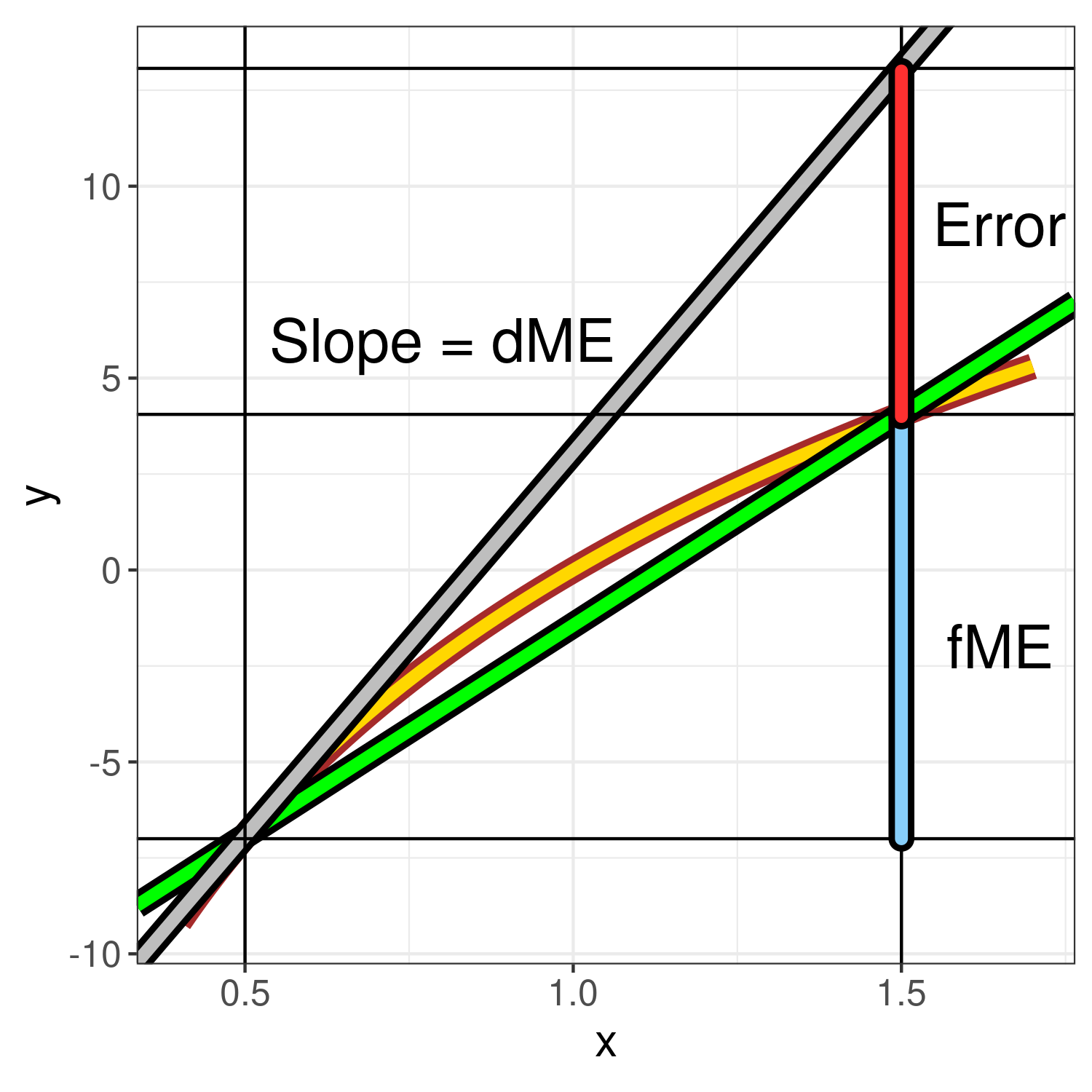}
\caption{The prediction function is yellow colored. The dME is given by the slope of the tangent (grey) at the point of interest (x = 0.5). The interpretation of the dME corresponds to the evaluation of the tangent value at x = 1.5, which is subject to an error (red) almost as large as the actual change in prediction. The fME equals the change in prediction along the secant (green) through the prediction at x = 0.5 and at x = 1.5.
\label{fig:ME_secant}}
\end{figure}

\subsection{Marginal Effects for Tree-Based Prediction Functions}

dMEs are not suited for interpreting piecewise constant prediction functions, e.g., classification and regression trees (CART) or tree ensembles such as random forests or gradient boosted trees. Generally, most observations are located on piecewise constant parts of the prediction function where the derivative equals zero. fMEs provide two advantages when interpreting tree-based prediction functions. First, a large enough step size will often involve traversing a jump discontinuity (a tree split in the case of CART) on the prediction function (see Fig. \ref{fig:tree_based_me}), so the fME does not equal zero. Second, when aggregating fME estimates, measures of spread such as the variance can indicate what fraction of fMEs traversed a jump discontinuity and what fraction did not. Multivariate fMEs (which we suggest in Section \ref{sec:novel_methodology}) further strengthen our argument, as traversing the prediction surface along multiple features has a higher chance of traversing jump discontinuities.

\begin{figure}
\centering
 \includegraphics[width=0.5\textwidth]{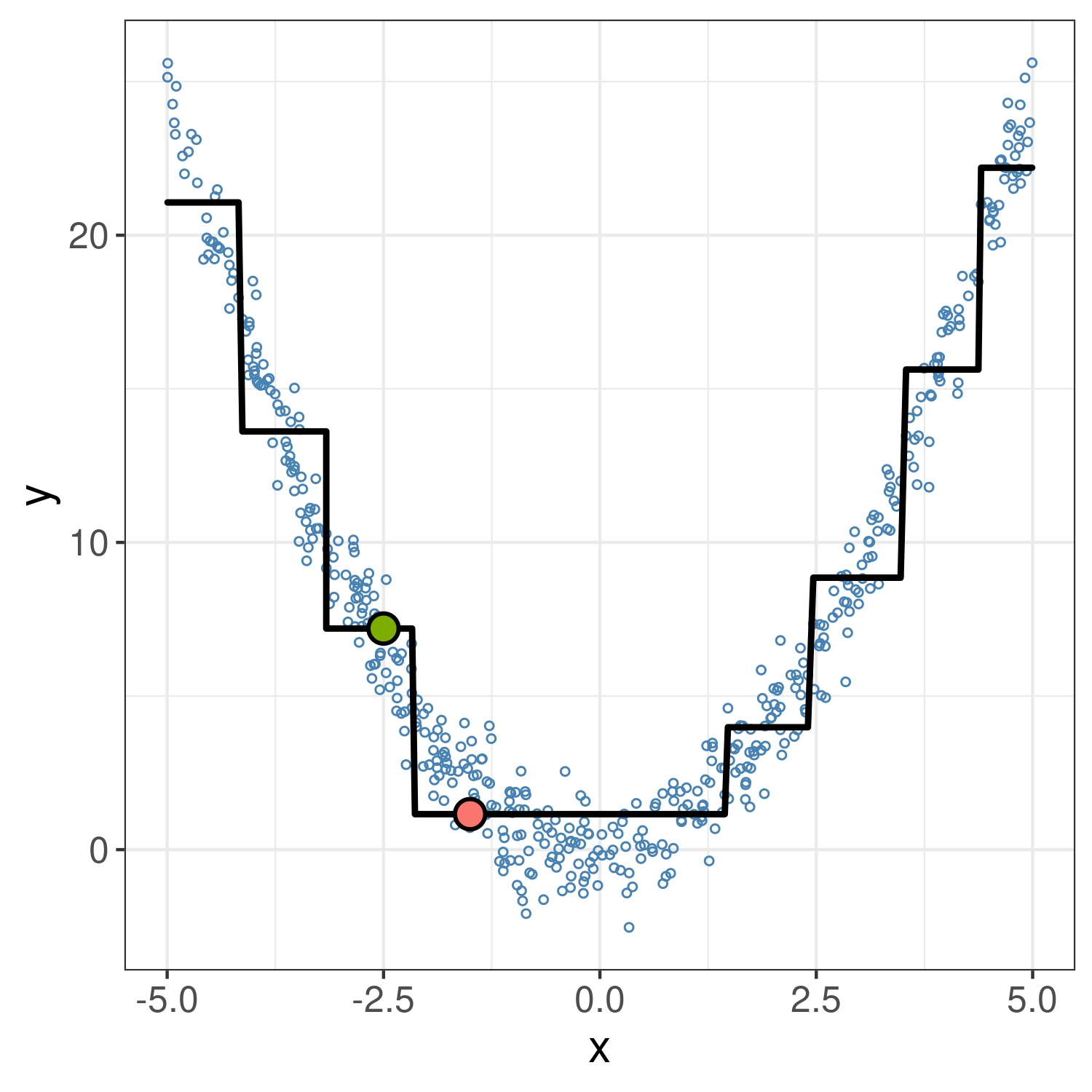}
\caption{\label{fig:tree_based_me}Consider a quadratic relationship between the target $y$ and a single feature $x$. A decision tree fits a piecewise constant prediction function (black line) to the training data (blue points). The dME at the point $x = -2.5$ (green dot) is zero, while the fME with $h = 1$ traverses the jump discontinuity and reaches the point $x = -1.5$ (red dot).}
\end{figure}

\subsection{Shortcomings of Existing Methodology}

First, in many cases, one wants to interpret the model based on changes in multiple features. The effect of multiple features may consist of univariate effects, as well as on interactions between features of various orders (see Appendix \ref{app:decomposition} for an introduction to decompositions of prediction functions). However, MEs are only capable of computing univariate feature effects. This inability to assess the combined effects of multiple features creates a barrier for their adoption. Second, the existing definition of categorical MEs in Eq. (\ref{eq:categorical_me}) is counter-intuitive and especially prone to extrapolations, as it replaces observed values on both sides of the forward difference with artificially constructed ones. Third, although usage of the fME results in an exact change in prediction, we lose information about the shape of the prediction function along the forward difference (see Fig. \ref{fig:ME_secant}). One therefore risks misinterpreting the fME as a linear effect. Fourth, the common practice to aggregate MEs to an AME, which acts as an estimate of the global feature effect, is misleading for non-linear prediction functions. Averaging may even result in canceling out individual effects with opposite signs, falsely suggesting that there is no influence of the feature on the target. We address these shortcomings in the following section.

\section{A New Class of Forward Marginal Effects}
\label{sec:novel_methodology}

This section introduces a new class of fMEs. First, we address the inadequate conception of categorical MEs by introducing a new, intuitive definition. Moreover, we extend univariate fMEs to multivariate changes in feature values. Next, we introduce the NLM for fMEs. The concept of estimating the expected ME through the AME is extended to a semi-global cME, which is conditional on a feature subspace. Lastly, we present ways to estimate cMEs through partitioning the feature space.

\subsection{Observation-Wise Categorical Marginal Effect}

Recall that the common definition of categorical MEs is based on first changing all observations' value of $x_j$ to each category and then computing the difference in predictions when changing it to the reference category. However, one is often interested in prediction changes if aspects of an actual observation change. We therefore propose an observation-wise categorical ME. We first select a single reference category $c_h$. For each observation whose feature value $x_j \neq c_h$, we predict once with the observed value $x_j$ and once where $x_j$ has been replaced by $c_h$:
\begin{equation}
\text{ME}_{\boldx, c_h} = \widehat{f}(c_h, \boldxMinusj) - \widehat{f}(\boldx) \quad , \quad x_j \neq c_h
\label{eq:new_categorical_me}
\end{equation}
This definition of categorical MEs is in line with the definition of fMEs, as we receive a single ME for a single observation with the observed feature value as the reference point. For $k$ categories, we receive $k-1$ sets of observation-wise categorical MEs. Keep in mind that changing categorical values also creates the risk of model extrapolations. However, observation-wise categorical MEs reduce the potential for extrapolations, as only one prediction of the FD is based on manipulated feature values.
\par
As an example, consider the Boston housing data (see Table \ref{tab:boston_housing_data}). We could be interested in the effect of changing the categorical feature \texttt{chas} (indicating whether a district borders the Charles river in Boston) on a district's median housing value in US dollars. For a district that does not border the river (\texttt{chas} = 0), the observation-wise categorical ME would investigate the effect of it bordering the river and vice versa.
\par
Observation-wise categorical MEs are also suited for computing a categorical AME. Note that we have a smaller number of categorical MEs, depending on the marginal distribution of $x_j$, which may affect the variance of the mean. However, observation-wise categorical MEs are not suited for MEMs, as we lose track of the observed values when averaging $\boldxMinusS$. The same partially holds for MERs, i.e., their computation is possible, but the MER obfuscates the interpretation we wish to have for observation-wise categorical MEs.

\subsection{Forward Differences with Multivariate Feature Value Changes}
\label{sec:multivariate_fme}

We can extend the univariate fME to an arbitrary set of features, which allows us to explore the surface of the prediction function in various directions simultaneously. Consider a multivariate change in feature values of a feature subset $S = \{1, \; \dots \; , s\}$ with $S \subseteq P$. The fME corresponds to Eq. (\ref{eq:multivariate_fME}). We denote the multivariate change $(x_1 + h_1, \; \dots \; , x_s + h_s)$ by $(\boldxS + \boldhS)$:
\begin{align}
\text{fME}_{\boldx, \boldhS} &= \widehat{f}(x_1 + h_1, \; \dots \; , x_s + h_s, \boldxMinusS) - \widehat{f}(x_1, \; \dots \; , x_s, \boldxMinusS)
\label{eq:multivariate_fME} \\
&= \widehat{f}(\boldxS + \boldhS, \boldxMinusS) - \widehat{f}(\boldx)\nonumber
\end{align}
It is straightforward to extend the univariate definition of the AME, MEM, and MER to multivariate feature value changes.
\par
One can demonstrate that the multivariate fME is equivalent to a difference in APs, which is used as an alternative to MEs. Consider two observations $\boldx$ and $\boldx^* = (x_1^*, \dots, x_p^*)$:
\begin{align*}
(x_1^* - x_1, \dots, x_p^* - x_p) &= (h_1, \dots, h_p) \\
\Leftrightarrow (x_1 + h_1, \dots, x_p + h_p) &= (x_1^*, \dots, x_p^*) \\
\end{align*}
It follows:
\begin{align*}
    \widehat{f}(\boldx^*) - \widehat{f}(\boldx) &= \widehat{f}(x_1 + h_1, \dots, x_p + h_p) - \widehat{f}(\boldx) \\
    &= \text{fME}_{\boldx, (h_1, \dots, h_p)}
\end{align*}
The multivariate fME bridges the gap between MEs and differences in APs, thus making a distinction obsolete.
\par
A technique with the additive recovery property only \textit{recovers} terms of the prediction function that depend on the feature(s) of interest $\boldxS$ or consist of interactions between the feature(s) of interest and other features, i.e., the method recovers no terms that exclusively depend on the remaining features $\boldxMinusS$ \citep{apley_ale}. In Appendix \ref{app:additiverecovery}, we derive the additive recovery property for fMEs.

\subsection{Relation between Forward Marginal Effects, the Individual Conditional Expectation, and Partial Dependence}

Given a data point $\boldx$, the ICE of a feature set $S$ corresponds to the prediction as a function of replaced values $\boldxS^{*}$ where $\boldxMinusS$ is kept constant:
\begin{equation*}
\text{ICE}_{\boldx, S}(\boldxS^{*}) = \widehat{f}(\boldxS^{*}, \boldxMinusS)
\end{equation*}
The PD on a feature set $S$ corresponds to the expectation of $\widehat{f}(\boldX)$ w.r.t. the marginal distribution of $\bm{X_{-S}}$. It is estimated via Monte Carlo integration where the draws $\boldxMinusS$ correspond to the sample values:
\begin{equation*}
\text{PD}_{\mathcal{D}, S}(\boldxS) = \frac{1}{n} \sum_{i = 1}^n \widehat{f}\left(\boldxS, \boldxMinusS^{(i)}\right) \\
\end{equation*}
We can visually demonstrate that in the univariate case, the fME is equivalent to the vertical difference between two points on an ICE curve. However, the AME is only equivalent to the vertical difference between two points on the PD curve for linear prediction functions (see Fig. \ref{fig:me_ice_relation}). We generalize this result to the multivariate fME and ICE , as well as the multivariate forward AME and PD (see Theorem \ref{theorem:equivalence_fme_ice} and Theorem \ref{theorem:equivalence_ame_pd} in Appendix \ref{app:proofs_fme_ice_pd}).
Visually assessing changes in prediction due to a change in feature values is difficult to impossible in more than two dimensions. High-dimensional feature value changes therefore pose a natural advantage for fMEs as opposed to techniques such as the ICE, PD, or ALEs, which are mainly interpreted visually.
\begin{figure}
\centering
 \includegraphics[width=0.5\textwidth]{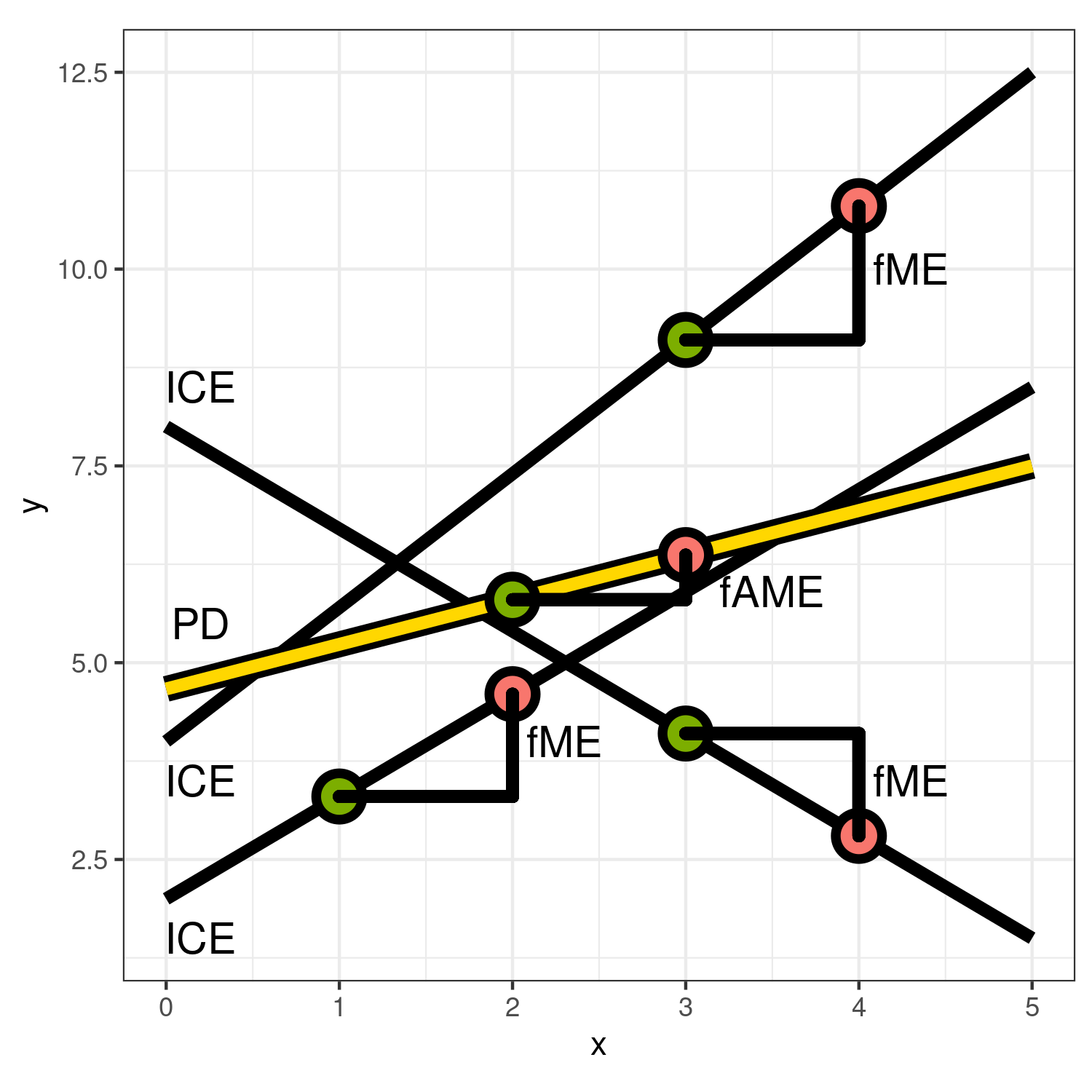}
\caption{\label{fig:me_ice_relation} Three ICE curves are black colored. The PD is the average of all ICE curves and is yellow colored. For each ICE curve, we have a single observation, visualized by the corresponding green dot. We compute the fME at each observation with a step size of 1, which results in the corresponding red dot. The fMEs are equivalent to the vertical difference between two points on the ICE curves. If the prediction function is linear in the feature of interest, the average of all fMEs is equivalent to the vertical difference between two points on the PD.
}
\end{figure}

\subsection{Non-Linearity Measure}
\label{sec:nlm}

Although an fME represents the exact change in prediction and always accurately describes the movement on the prediction function, we lose information about the function's shape along the forward difference. It follows that when interpreting fMEs, we are at risk of misjudging the shape of the prediction function as a piecewise linear function. However, prediction functions created by ML algorithms are not only non-linear but also differ considerably in shape across the feature space. We suggest to augment the change in prediction with an NLM that quantifies the deviation between the prediction function and a linear reference function. First, the fME tells us the change in prediction for pre-specified changes in feature values. Then, the NLM tells us how accurately a linear effect resembles the change in prediction. 
\par
Given only univariate changes in feature values, we may visually assess the non-linearity of the feature effect with an ICE curve. However, the NLM quantifies the non-linearity in a single metric, which can be utilized in an informative summary output of the prediction function. In Section \ref{sec:cAME}, we estimate feature effects conditional on specific feature subspaces. The individual NLM values can be used to more accurately describe the feature subspace. Given changes in more than two features, visual interpretation techniques such as the ICE and PD are not applicable. As opposed to this, the NLM is defined in arbitrary dimensions.

\subsubsection{Linear Reference Function}

A natural choice for the linear reference function is the secant intersecting both points of the forward difference (see Fig. \ref{fig:deviation_measure}).
\begin{figure}
\centering
\includegraphics[width=0.49\linewidth]{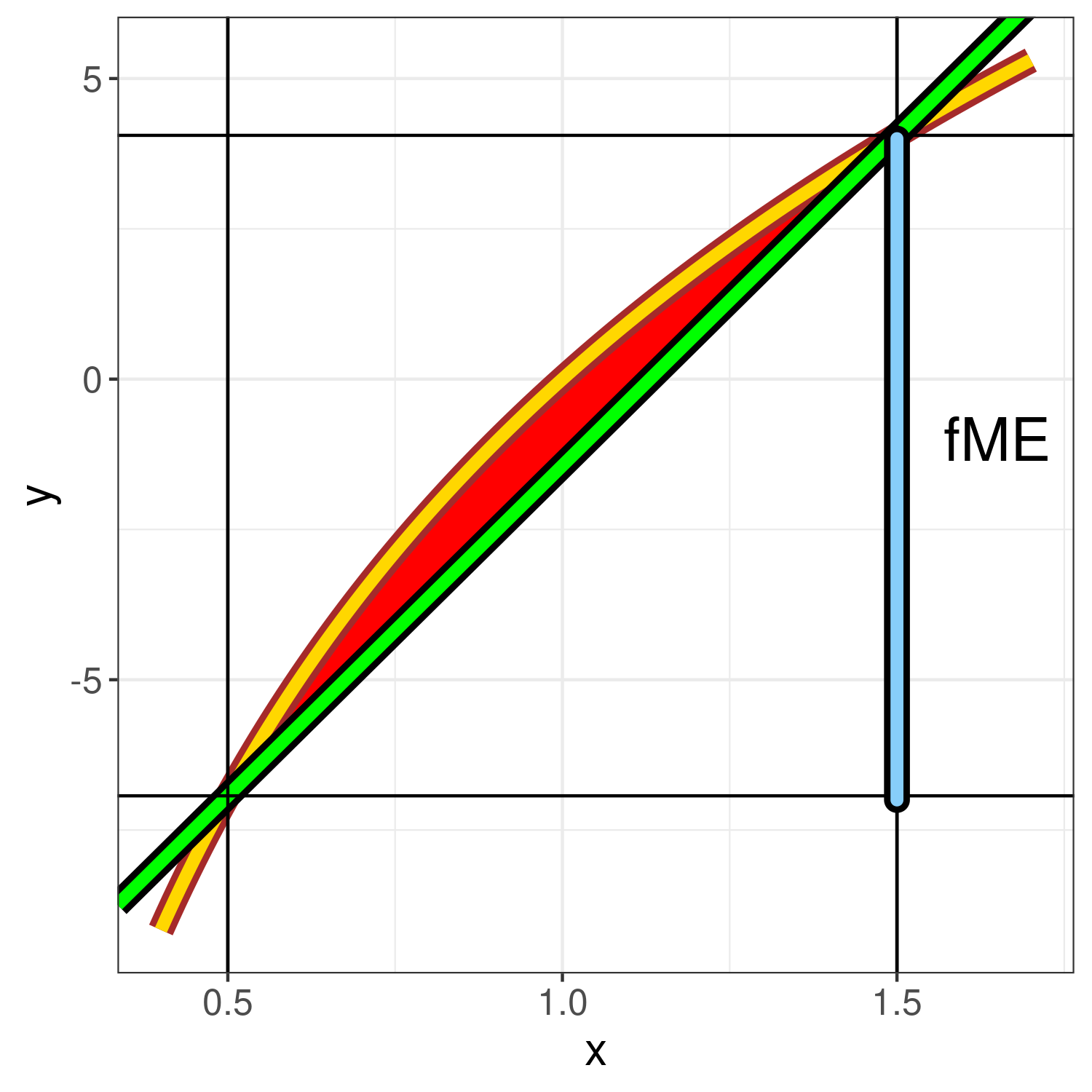}
\caption{We wish to determine the area between the prediction function and the secant. \label{fig:deviation_measure}
}
\end{figure}
Using the multivariate secant as the linear reference is based on the assumption that a multivariate change in feature values is a result from equally proportional changes in all features. Although this assumption is somewhat restrictive, it allows for a great simplification of the non-linearity assessment, as the secant is unique and easy to evaluate. The secant for a multivariate fME corresponds to Eq. (\ref{eq:geodesic}):
\begin{align}
&\phantom{{}={}} g_{\boldx, \boldhS}(t) = \begin{pmatrix} x_1 + t \cdot h_1 \\ \vdots \\ x_s + t \cdot h_s \\ \vdots \\ x_p \\ \widehat{f}(\boldx) + t \, \cdot \, \text{fME}_{\boldx, \boldhS} \end{pmatrix} \label{eq:geodesic}
\end{align}
Fig. \ref{fig:deviation_measure_multivariate} visualizes the discrepancy between the prediction function and the secant along a two-dimensional fME. If the NLM indicates linearity, we can infer that if \textit{all} individual feature changes are multiplied by a scalar $t \in [0, 1]$, the fME would change by $t$ as well.

\begin{figure}
\centering
 \includegraphics[width=0.49\textwidth]{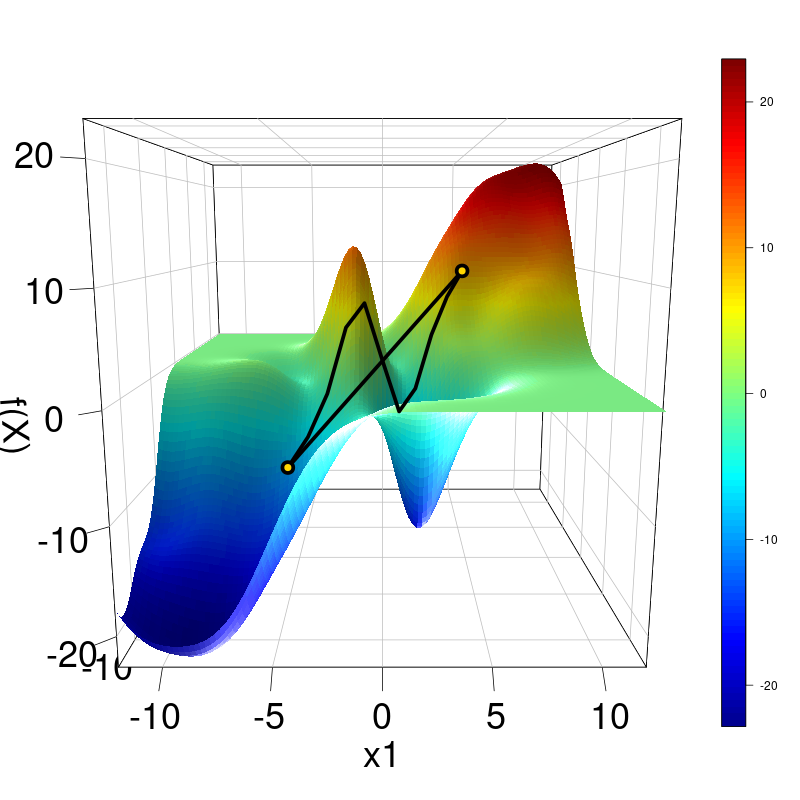}
 \includegraphics[width=0.49\textwidth]{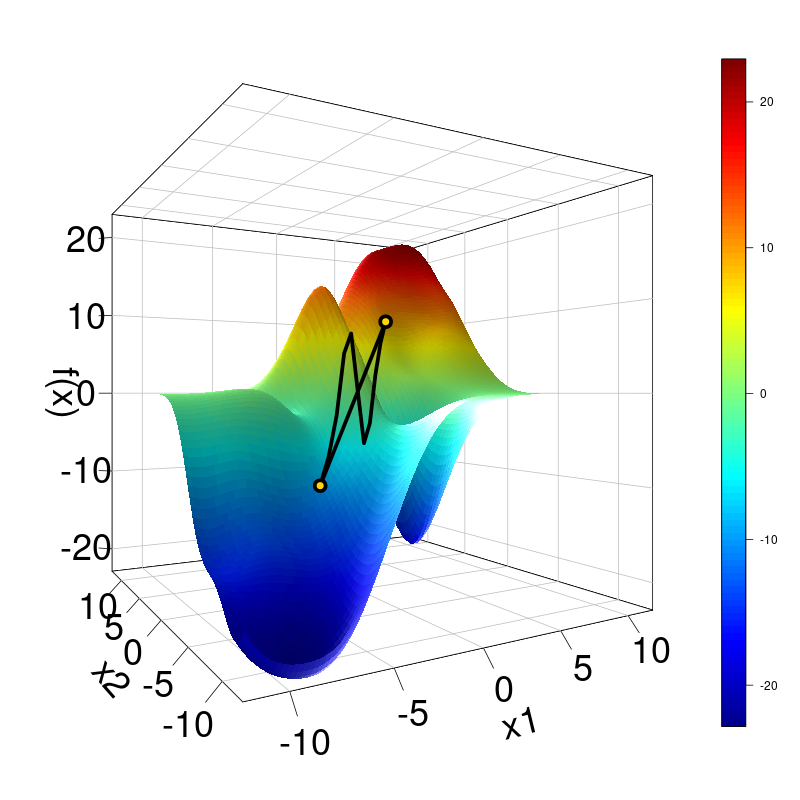}
\caption{\label{fig:deviation_measure_multivariate} A non-linear prediction function, the path along its surface, and the corresponding secant along a two-dimensional fME from point (-5, -5) to point (5, 5).}
\end{figure}

\subsubsection{Definition and Interpretation}

Comparing the prediction function against the linear reference function along the fME is a general concept that can be implemented in various ways. One requires a normalized metric that indicates the degree of similarity between functions or sets of points.
The Hausdorff \citep{belogay_hausdorff_distance} and Fréchet \citep{alt_frechet} distances are established metrics that originate in geometry. Another option is to integrate the absolute or squared deviation between both functions. These approaches have the common disadvantage of not being normalized, i.e., the degree of non-linearity is scale-dependent.
\par
\citet{molnar_complexity} compare non-linear function segments against linear models via the coefficient of determination $R^2$. In this case, $R^2$ indicates how well the linear reference function is able to explain the non-linear prediction function compared to the most uninformative baseline model, i.e., one that always predicts the prediction function through its mean value. As we do not have observed data points along the forward difference, points would need to be obtained through (Quasi-)Monte-Carlo sampling, whose error rates heavily depend on the number of sampled points. 
As both the fME and the linear reference function are evaluated along the same single path across the feature space, their deviation can be formulated as a line integral. Hence, we are able to extend the concept of $R^2$ to continuous integrals, comparing the integral of the squared deviation between the prediction function and the secant, and the integral of the squared deviation between the prediction function and its mean value. The line integral is univariate and can be numerically approximated with various techniques such as Gaussian quadrature.
\par
The parametrization of the path through the feature space is given by $\gamma: [0, 1] \mapsto \mathcal{X}$, where $\gamma(0) = \boldx$ and $\gamma(1)= (\boldxS + \boldhS, \boldxMinusS)$. The line integral of the squared deviation between prediction function and secant along the forward difference corresponds to:
\begin{align*}
     \text{(I)} = \int_0^1 \left(\widehat{f}(\gamma(t)) - g_{\boldx, \boldhS}(\gamma(t))\right)^2 \; \Big\vert \Big\vert \frac{\partial \gamma(t)}{\partial t}\Big\vert \Big\vert_2 \;dt
\end{align*}
with 
\begin{align*}
    \gamma(t) = \begin{pmatrix} x_1 \\ \vdots \\ x_p \end{pmatrix} + t \cdot \begin{pmatrix} h_1 \\ \vdots \\ h_s \\ 0 \\ \vdots \\ 0  \end{pmatrix}
    \quad  , \quad t \in [0, 1] 
\end{align*}
and
\begin{align*}
    \Big\vert \Big\vert \frac{\partial \gamma(t)}{\partial t} \; \Big\vert \Big\vert_2 &= \sqrt{h_1^2 + \dots + h_s^2}
\end{align*}
The integral of the squared deviation between the prediction function and the mean prediction is used as a baseline. The mean prediction is given by the integral of the prediction function along the forward difference, divided by the length of the path:
\begin{align*}
    \overline{\widehat{f}} &= \frac{\int_0^1 \widehat{f}(\gamma(t)) \; \Big\vert \Big\vert \frac{\partial \gamma(t)}{\partial t}\Big\vert \Big\vert_2 \;dt}{\int_0^1 \Big\vert \Big\vert \frac{\partial \gamma(t)}{\partial t} \Big\vert \Big\vert_2 \; dt}  \\
    &= \int_0^1 \widehat{f}(\gamma(t)) \;dt
\end{align*}

\begin{align*}
    \text{(II)} = \int_0^1 \widehat{f}\left(\gamma(t)) - \overline{\widehat{f}}\right)^2 \; \Big\vert \Big\vert \frac{\partial \gamma(t)}{\partial t}\Big\vert \Big\vert_2 \;dt
\end{align*}
The $\text{NLM}_{\boldx, \boldhS}$ is defined as:
$$
\text{NLM}_{\boldx, \boldhS} = 1 - \frac{\text{(I)}}{\text{(II)}}
$$

\par
The values of the NLM have an upper limit of 1 and indicate how well the secant can explain the prediction function, compared to the baseline model of using the mean prediction. For a value of 1, the prediction function is equivalent to the secant. A lower value indicates an increasing non-linearity of the prediction function. For negative values, the mean prediction better predicts values on the prediction function than the secant, i.e., the prediction function is highly non-linear. We suggest to use 0 as a hard bound to indicate non-linearity and values on the interval $]0, 1[$ as an optional soft bound.
\par
Instead of conditioning on pre-specified feature changes, we may also reverse the procedure and optimize the NLM by adjusting the step sizes. For instance, given a patient with certain health characteristics, we could search for the maximum step sizes in \texttt{age} and \texttt{weight} that correspond to an approximately linear fME on the disease risk, e.g., with an NLM $\geq 0.9$.

\subsection{Conditional Forward Marginal Effects on Feature Subspaces}
\label{sec:cAME}

It is desirable to summarize the feature effect in a single metric, similarly to the parameter-focused interpretation of linear models. For instance, one is often interested in the expected ME, which can be estimated via the AME. However, averaging heterogeneous MEs to the AME is not globally representative of non-linear prediction functions such as the ones created by ML algorithms. A heterogeneous distribution of MEs requires a more local evaluation. As opposed to conditioning on feature values in the case of MERs (local), we further suggest to condition on specific feature subspaces (semi-global).
We can partition the feature space into mutually exclusive subspaces (see Fig. \ref{fig:hyperrectangles}) where MEs are more homogeneous. In such a feature subspace, we can average the fMEs to a conditional AME (cAME), which is an estimate of the conditional ME (cME). The feature space partitioning, including the cAMEs, serves as a compact semi-global summary of the prediction function or as a descriptor of population subgroups. Subgroup analysis is of interest in a multitude of application contexts \citep{atzmueller_subgroups}. The cME and cAME (conditional on the feature subspace $\mathcal{X}_{[\;]}$) correspond to:
\begin{align*}
\text{cME}_{\boldX_{[\;]}, \boldhS} &= \mathbb{E}_{\boldX_{[\;]}} \left[\text{fME}_{\boldX_{[\;]}, \boldhS} \right] \\
\text{cAME}_{\mathcal{D}_{[\;]}, \boldhS} &= \frac{1}{n} \sum_{i = 1}^n \left(\widehat{f}\left(\boldxS^{(i)} + \boldhS, \boldxMinusS^{(i)}\right) - \widehat{f}\left(\boldxi\right) \right) \quad \forall i: \boldxi \in \mathcal{X}_{[\;]}\\
\end{align*}

\begin{figure}
\centering
 \includegraphics[width=0.4\linewidth]{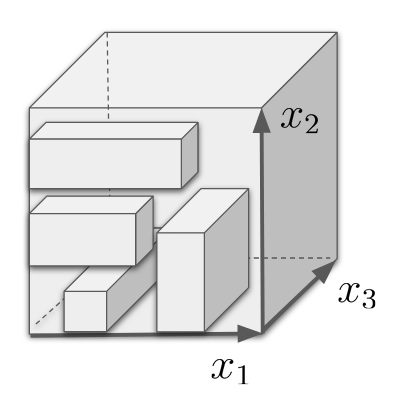}
\caption{An exemplary three-dimensional feature space is partitioned into four hyperrectangles of different sizes. Our goal is to partition the entire feature space into hyperrectangles such that the variance of MEs on each subspace is minimized.\label{fig:hyperrectangles}}
\end{figure}

\subsubsection{Partitioning the Feature Space}

Decision tree learning is an ideal scheme to partition the entire feature space into mutually exclusive subspaces. Growing a tree by global optimization poses considerable computational difficulties and corresponds to an NP-complete problem \citep{norouzi_nongreedy_tree}. Recent developments in computer science and engineering can be explored to revisit global decision tree optimization from a different perspective, e.g., \citet{bertsimas_global_trees} explore mixed-integer optimization (which has seen speedups in the billion factor range over the last decades) to find globally optimal decision trees. To reduce computational complexity, the established way (which is also commonly available through many software implementations) is through recursive partitioning (RP), optimizing an objective function in a greedy way for each tree node.
\par
Over the last decades, a large variety of RP methods has been proposed \citep{loh_trees_review}, with no gold standard having crystallized to date. In principle, any RP method that is able to process continuous targets can be used to find cAMEs, e.g., classification and regression trees (CART) \citep{breiman_cart, hastie_elemstatlearn}, which is one of the most popular approaches. Trees have been demonstrated to be notoriously unstable w.r.t. perturbations in input data \citep{zhou_approximation_trees, mark_tree_instability}. Tree ensembles, such as random forests \citep{breiman_randomforests}, reduce variance but lose interpretability as a single tree structure. Exchanging splits along a single path results in structurally different but logically equivalent trees \citep{turney_tree_stability}. 
It follows that two structurally very distinct trees can create the same or similar subspaces. We are therefore not interested in the variance of the tree itself, but in the subspaces it creates.
\par
Subspace instabilities affect the interpretation, e.g., when trying to find population subgroups. One should therefore strive to stabilize the found subspaces, e.g., by stabilizing the splits.
A branch of RP methods incorporates statistical theory into the split procedure. Variants include conditional inference trees (CTREE) \citep{hothorn_ctree}, which use a permutation test to find statistically significant splits; model-based recursive partitioning (MOB) \citep{zeileis_mob}, which fits node models and tests the instability of the model parameters w.r.t. partitioning the data; or approximation trees \citep{zhou_approximation_trees}, which generate artificially created samples for significance testing of tree splits. \citet{seibold_subgroups} use MOB to find patient subgroups with similar treatment effects in a medical context.
The variance and instability of decision trees partly stems from binary splits, as a decision higher up cascades through the entire tree and results in different splits lower down the tree \citep{hastie_elemstatlearn}. Using multiway trees, which also partition the entire feature space, would therefore improve stability. However, multiway splits are associated with a considerable increase in computational complexity and are therefore often discarded in favor of binary splitting \citep{zeileis_mob}. For the remainder of the paper, we use CTREE to compute cAMEs.

\subsubsection{Interpretation and Confidence Intervals}

As the SSE, and therefore also the variance and standard deviation (SD), of the fMEs is highly scale-dependent, the regression tree can be interpreted via the coefficient of variation (CoV). The CoV, also referred to as the relative SD, corresponds to the SD divided by the (absolute) mean and is scale-invariant. In our case, it is the SD of the fMEs divided by the absolute cAME inside each subspace. We strive for the lowest value possible, which corresponds to the highest homogeneity of fMEs.
Given a cME estimate, it is desirable to estimate the conditional NLM (cNLM) on the corresponding subspace. It is trivial to include a conditional average NLM (cANLM) in the feature subspace summary, i.e., we simply average all NLMs on the found subspaces. The cANLM gives us an estimate of the expected non-linearity of the prediction function for the given movements along the feature space, inside the specific feature subspace.
\par
Fig. \ref{fig:ame_tree_split} visualizes an exemplary feature space partitioning for the Boston housing data on the fMEs resulting from increasing the average number of rooms per dwelling (feature \texttt{rm}) by 1. The global AME is 5.43 with a CoV of 0.68. The leftmost leaf node contains 113 observations with a cAME of 6.79 (CoV = 0.31), i.e., the homogeneity of the fMEs is considerably higher than on the entire feature space. The average linearity of the fMEs on this subspace is high (cANLM = 0.93 with CoV = 0.36). Given the global ranges of all features, every decision rule can be transformed into a compact summary of the feature space, such as in Table \ref{tab:ame_subspace}.

\begin{figure}
\centering
 \includegraphics[width = \textwidth]{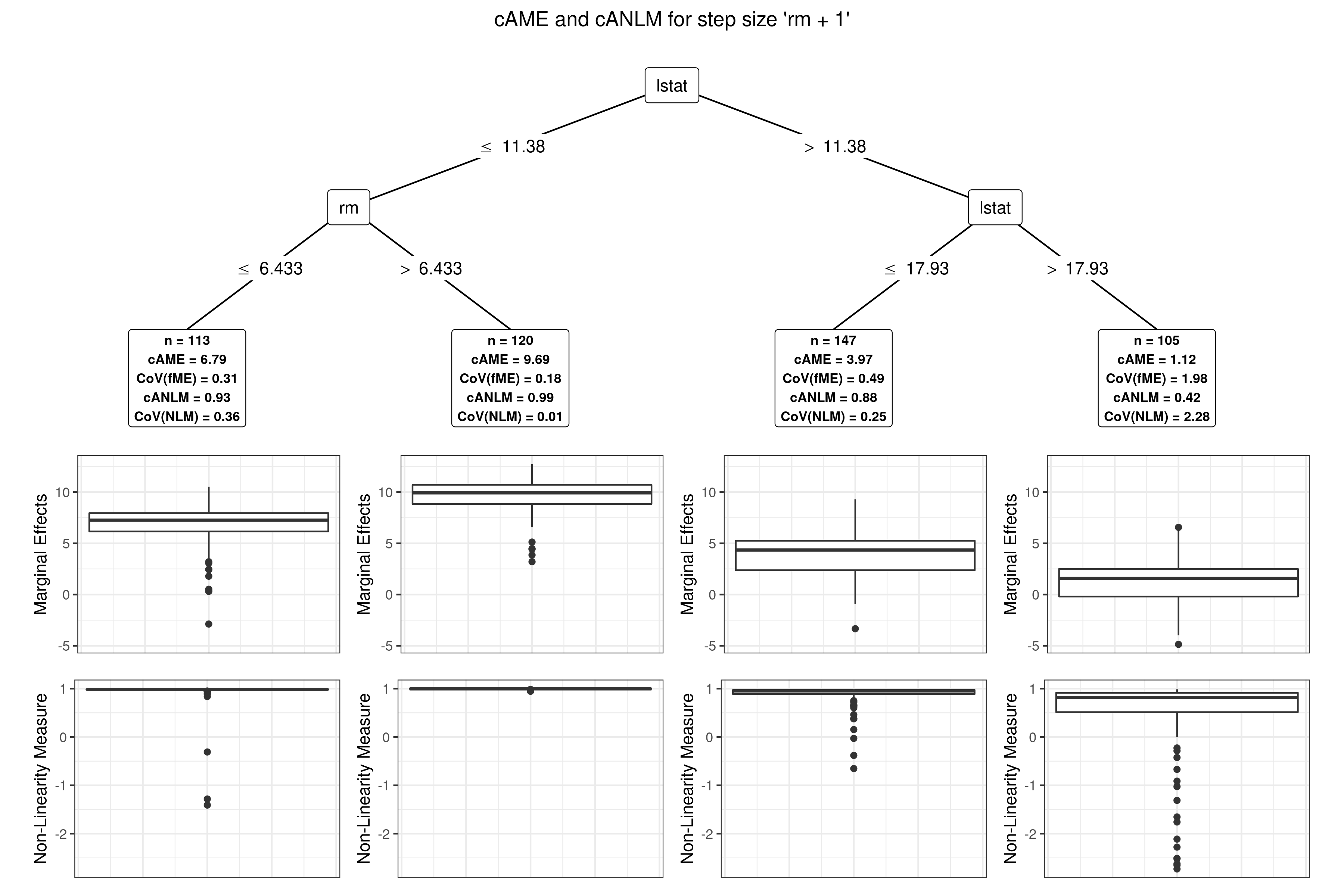}
 \caption{\label{fig:ame_tree_split} We use CTREE to recursively partition the feature space so that the homogeneity of fMEs is maximized on each subspace. The average of the fMEs on each subspace is the corresponding cAME. For example, the leftmost leaf node contains 113 observations with a cAME of 6.79 and a cANLM of 0.93, i.e., the average linearity of the fMEs on this subspace is high. The respective CoV values indicate a high homogeneity of the fME and NLM values.
 The leftmost leaf node is summarized in Table \ref{tab:ame_subspace}.
 }
\end{figure}
\begin{table}
\caption{\label{tab:ame_subspace}One of 4 leaf nodes (the leftmost) of a decision tree (visualized in Fig. \ref{fig:ame_tree_split}) that partitions the feature space into mutually exclusive subspaces where MEs are more homogeneous compared to the entire feature space. The features whose range is restricted to span the subspace are emphasized.}
\centering
\begin{tabular}[t]{|l l || l l|l l|}
\hline
 n & 113 & age & [2.9, 100] & dis  & [1.1296, 12.1265]\\
 cAME & 6.79 & zn & [0, 100] & rad & [1, 24] \\
 CoV(fME) & 0.31 & indus  & [0.46, 27.74] &  tax  & [187, 711] \\
 cANLM & 0.93 & chas & \{0, 1\} &  afram & [0.32, 396.9]\\
 CoV(NLM) & 0.36 & nox & [0.385, 0.871] &  ptratio & [12.6, 22]\\
 & & \textbf{rm}  & \textbf{[3.561, 6.433]} &  \textbf{lstat} & \textbf{[1.73, 11.38]} \\
 & & crim & [0.00632, 88.9762] & & \\
\hline
\end{tabular}
\end{table}

\par
The goal of finding feature subspaces where the feature effect is homogeneous essentially corresponds to finding population subgroups. Given the subspace, we may reverse the analysis as follows. If we sampled an observation on this feature subspace (which may be quite large in hypervolume and comprise many different possible data points), what is the expected change in predicted outcome given the specified changes in feature values, what is the expected non-linearity of the expected change in predicted outcome, and how confident are we in our estimates? The first question is addressed by the cAME, the second by the cANLM, and the third by the SD and the number of observations. Lower SD values increase confidence in the estimate and vice versa, and a larger number of observations located in the subspace increase confidence in the estimate and vice versa. Although we do not specify a distribution of the underlying fMEs or NLMs, constructing a confidence interval (CI) is possible via the central limit theorem. As the cAME and cANLM are sample averages of all fMEs and NLM values on each subspace, they both approximately follow a t-distribution (as the SD is estimated) for large sample sizes. Given a subspace $\mathcal{X}_{[\;]}$ that contains $n_{[\;]}$ observations, mean ($\text{cAME}_{\mathcal{D}_{[\;]}, \boldhS}$ and $\text{cANLM}_{\mathcal{D}_{[\;]}, \boldhS}$) and SD ($\text{SD}_{\text{fME}, \;[\;]}$ and $\text{SD}_{\text{NLM}, \;[\;]}$) values, the confidence level $\alpha$, and the values of the t-statistic with $n_{[\;]}-1$ degrees of freedom at the $1-\frac{\alpha}{2}$ percentile ($t_{1-\frac{\alpha}{2}, \; n_{[\;]}-1}$), the CIs correspond to:
\begin{align*}
\text{CI}_{\text{cAME}, \; 1 - \alpha} &= \left[\text{cAME}_{\mathcal{D}_{[\;]}, \boldhS} - t_{1-\frac{\alpha}{2}, \; n_{[\;]}-1} \frac{\text{SD}_{\text{fME}, \; [\;]}}{\sqrt{n}_{[\;]}} \; , \; \text{cAME}_{\mathcal{D}_{[\;]}, \boldhS} + t_{1-\frac{\alpha}{2}, \; n_{[\;]}-1} \frac{\text{SD}_{\text{fME}, \; [\;]}}{\sqrt{n_{[\;]}}}\right] \\
\text{CI}_{\text{cANLM}, \; 1 - \alpha} &= \left[\text{cANLM}_{\mathcal{D}_{[\;]}, \boldhS} - t_{1-\frac{\alpha}{2}, \; n_{[\;]}-1} \frac{\text{SD}_{\text{NLM}, \;[\;]}}{\sqrt{n}_{[\;]}} \; , \; \text{cANLM}_{\mathcal{D}_{[\;]}, \boldhS} + t_{1-\frac{\alpha}{2}, \; n_{[\;]}-1} \frac{\text{SD}_{\text{NLM}, \;[\;]}}{\sqrt{n_{[\;]}}}\right]
\end{align*} 
One option to ensure that the lower sample size threshold for CIs is valid is to specify a minimum node size for the computation of cAMEs, i.e., not growing the tree too large.
\par
When estimating the conditional feature effect for a hypothetical observation, it is important to avoid model extrapolations of the tree. We may create a hypothetical observation, use the tree to predict the cME estimate, and express our confidence in the estimate with a CI. However, the feature values of the observation may be located outside of the distribution of observations that were used to train the decision tree and therefore result in unreliable predictions.

\section{Comparison to LIME}
\label{sec:comparison_lime}

LIME \citep{ribeiro_lime}, one of the state-of-the-art model-agnostic feature effect methods (see Section \ref{sec:review}), most closely resembles the interpretation given by an fME. It also serves as a local technique, explaining the model for a single observation. LIME samples instances, predicts, and weighs the predictions by the instances' proximity to the instance of interest using a kernel function. Afterwards, an interpretable surrogate model is trained on the weighted predictions. The authors choose a sparse linear model, whose beta coefficients provide an interpretation similar to the fME.
\par
But there is a fundamental difference between both approaches. The fME directly works on the prediction function, while LIME trains a local surrogate model. The latter is therefore affected by an additional layer of complexity and uncertainty. The authors suggest to use LASSO regression, which requires choosing a regularization constant. Furthermore, one must select a similarity kernel defined on a distance function with a width parameter. The model interpretation is therefore fundamentally determined by multiple parameters.
Furthermore, certain surrogate models are incapable of explaining certain model behaviors and may potentially mislead the practitioner to believe the interpretation \citep{ribeiro_lime}. A linear surrogate model may not be able to describe extreme non-linearities of the prediction function, even within a single locality of the feature space. In contrast, the only parameters for the fME are the features and the step sizes. Without question, the choice of parameters for fMEs also significantly affects the interpretation. However, we argue that their impact is much clearer than in LIME, e.g., a change in a feature such as \texttt{age} is much more meaningful than a different width parameter in LIME. In fact, we argue that the motivation behind both approaches is fundamentally different. For fMEs, we start with a meaningful interpretation concept in mind, e.g., we may be interested in the combined effects of increasing \texttt{age} and \texttt{weight} on the disease risk. For LIME, we start with a single observation, trying to distill the black box model behavior within this specific locality into a surrogate model.
\par
In addition to the sensitivity of results regarding parameter choices \citep{slack_fooling_lime}, LIME is notoriously unstable even with fixed parameters. \citet{zhou_slime} note that repeated runs using the same explanation algorithm on the same model for the same observation results in different model explanations, and they suggest significance testing as a remedy. In contrast, fMEs with fixed parameters are deterministic.
\par
Furthermore, the authors of LIME note that the faithfulness of the local surrogate may be diminished by extreme non-linearities of the model, even within the locality of the instance of interest. This exact same critique holds for the fME (see Section \ref{sec:nlm}).  
Hence, we introduce the NLM, which essentially corresponds to a measure of faithfulness of the fME and whose concept can potentially be used for other methods as well. One could also use the coefficient of determination $R^2$ to measure the goodness-of-fit of the linear surrogate to the pseudo sample in LIME. However, we argue that the goodness-of-fit to a highly uncertain pseudo sample is a questionable way of measuring an explanation's faithfulness.
\par
The authors of LIME note that insights into the global workings of the model may be gained by evaluating multiple local explanations. As there usually are time constraints so that not all instances can be evaluated, an algorithm suggests a subset of representative instances. Although this approach avoids the issue of misrepresenting global effects by averaging local explanations, it also misses the opportunity to provide meaningful semi-global explanations. This is where the cAME comes into play. It is motivated by the goal to aggregate local interpretations while staying faithful to the underlying predictive model.

\section{Simulations}
\label{sec:simulations}

Here, we present three simulation scenarios to highlight the interplay between fMEs, the NLM, and the cAME. In all following sections, we use Simpson's 3/8 rule for the computation of the NLM and CTREE to compute cAMEs.

\subsection{Univariate Data without Noise}

We start with a univariate scenario without random noise and work directly with the data generating process (DGP). This way, we can evaluate how introducing noise affects the information gained from fMEs in the subsequent simulation. 
We simulate a single feature $x$, uniformly distributed on $[-5, 5]$, and define $f$ as:
$$
f(x) =
\begin{cases} x & x < 0 \\ 5 \sin (2x) & x \geq 0
\end{cases}
$$
The data is visualized in Fig. \ref{fig:simulation_univariate_data}. An fME with step size $h = 2$ is computed for each observation. We use CTREE on the fMEs to find cAMEs. Subsequently, all observations' NLM values are averaged to cANLM values on the subspaces of the cAMEs. Our computations are visualized in Fig. \ref{fig:simulation_univariate_no_noise_fme_nlm}. In the univariate case, we see a direct relationship between the shape of the DGP and the fMEs and NLM values. The NLM has ramifications on the interpretation of the fMEs. For instance, for $x = -3$, the fME of increasing $x$ by 2 units increases the predicted target value by 2 units, and we can conclude that the same holds proportionally for feature value changes of smaller magnitudes, e.g., a change of 1 unit results in an fME of 1, etc. On the contrary, given an observation $x = 1$, the NLM indicates considerable non-linearity. As a result, we cannot draw conclusions about this region of the feature space regarding fMEs with smaller step sizes than 2 units.
\begin{figure}
\centering
    \includegraphics[width=0.49\textwidth]{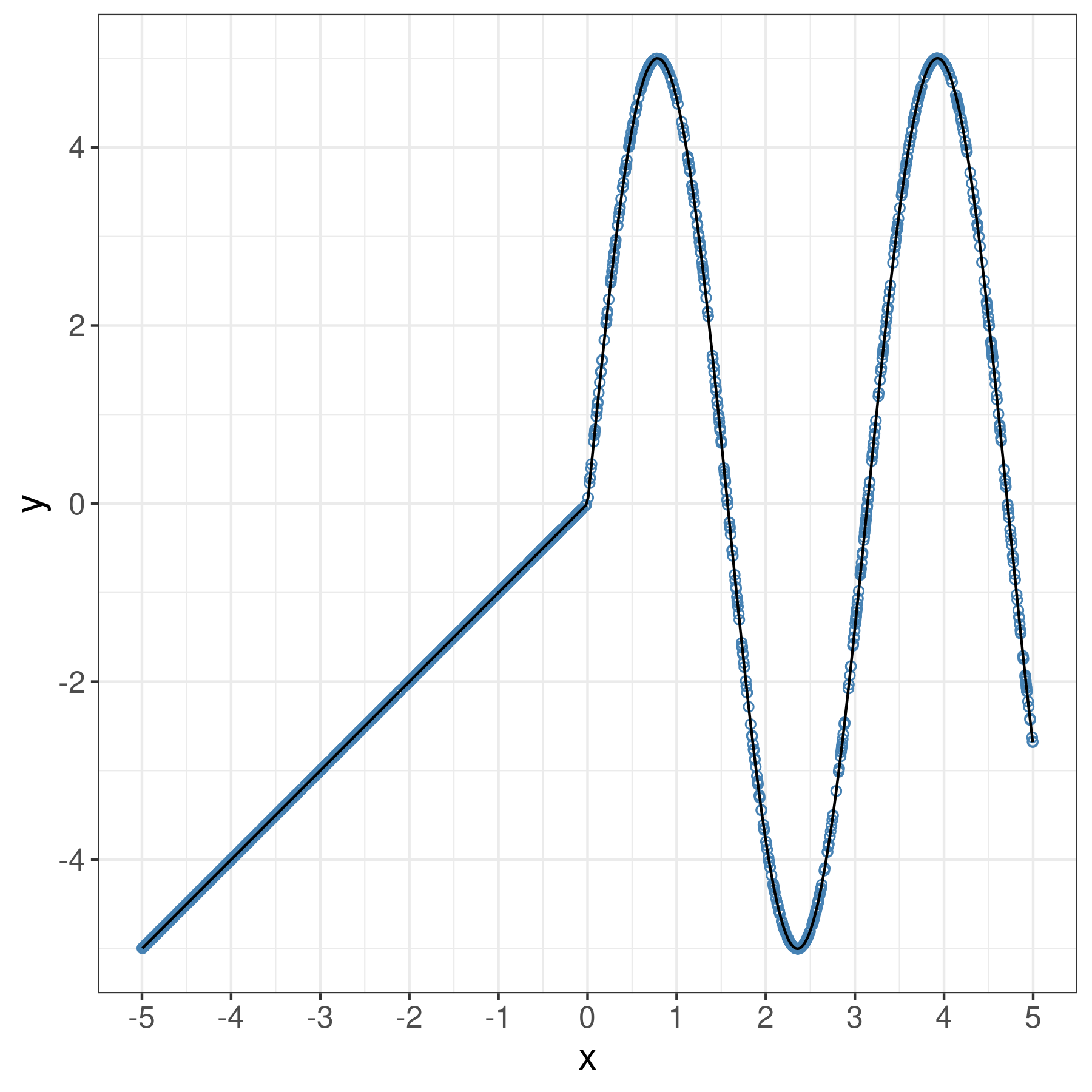}
   \includegraphics[width=0.49\textwidth]{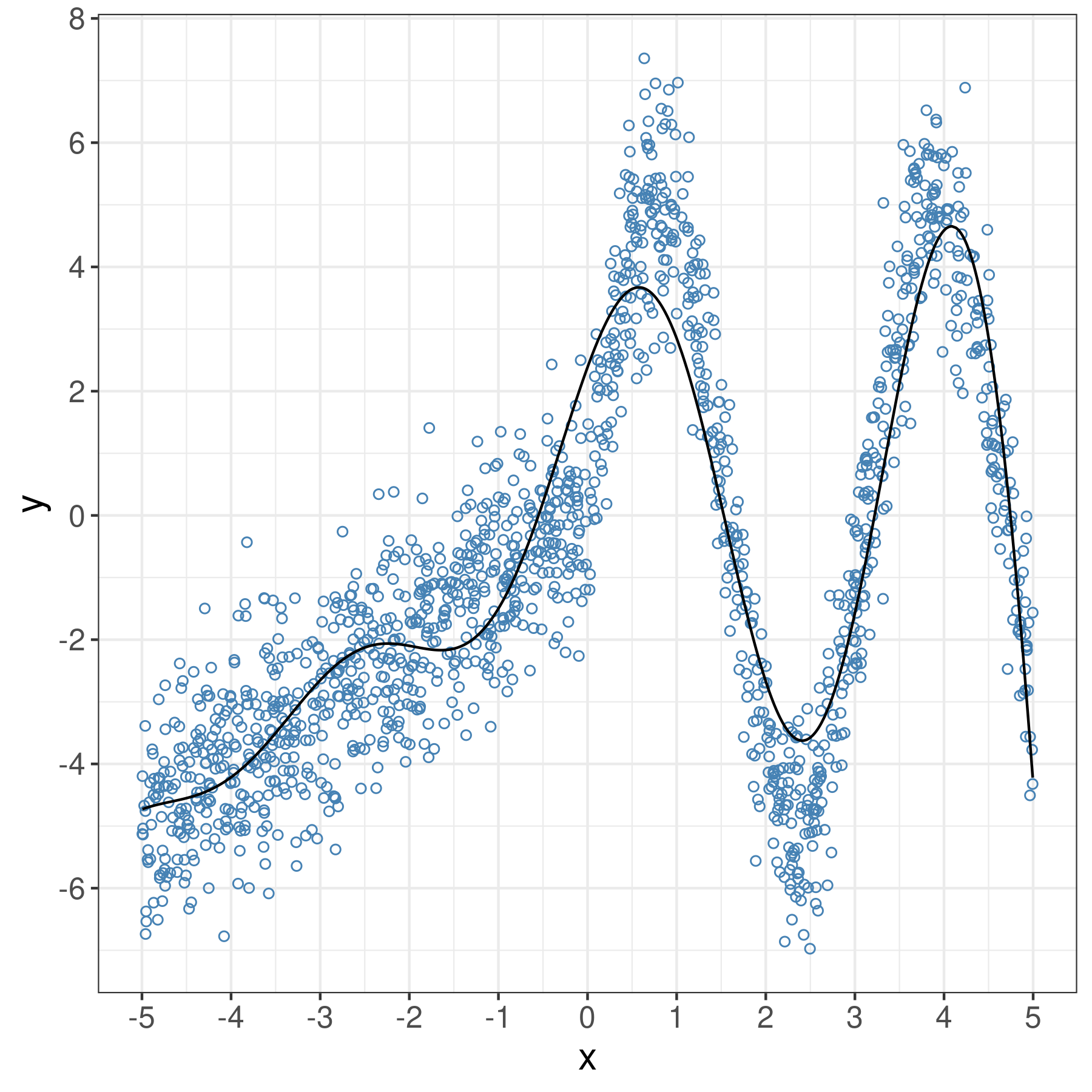}
\caption{\label{fig:simulation_univariate_data}The target is determined by a single feature $x$. On the interval $[-5, 0[$ there is a linear feature effect. On the interval $[0, 5]$ the functional relationship consists of a transformed sine wave. We first use the DGP, then add random noise on top of the data and train an SVM.}
\end{figure}
\begin{figure}
\centering
 \includegraphics[width=0.49\linewidth]{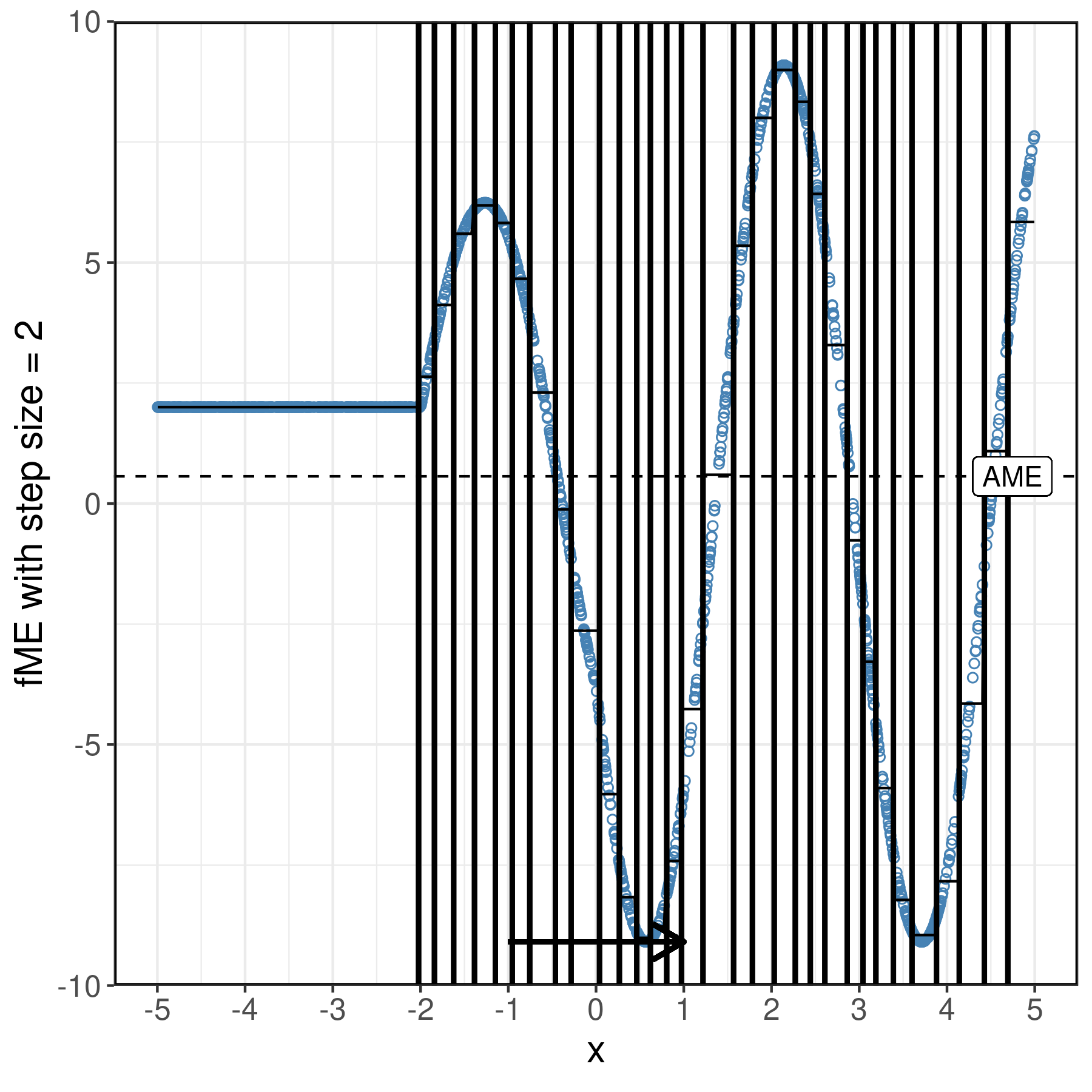}
  \includegraphics[width=0.49\linewidth]{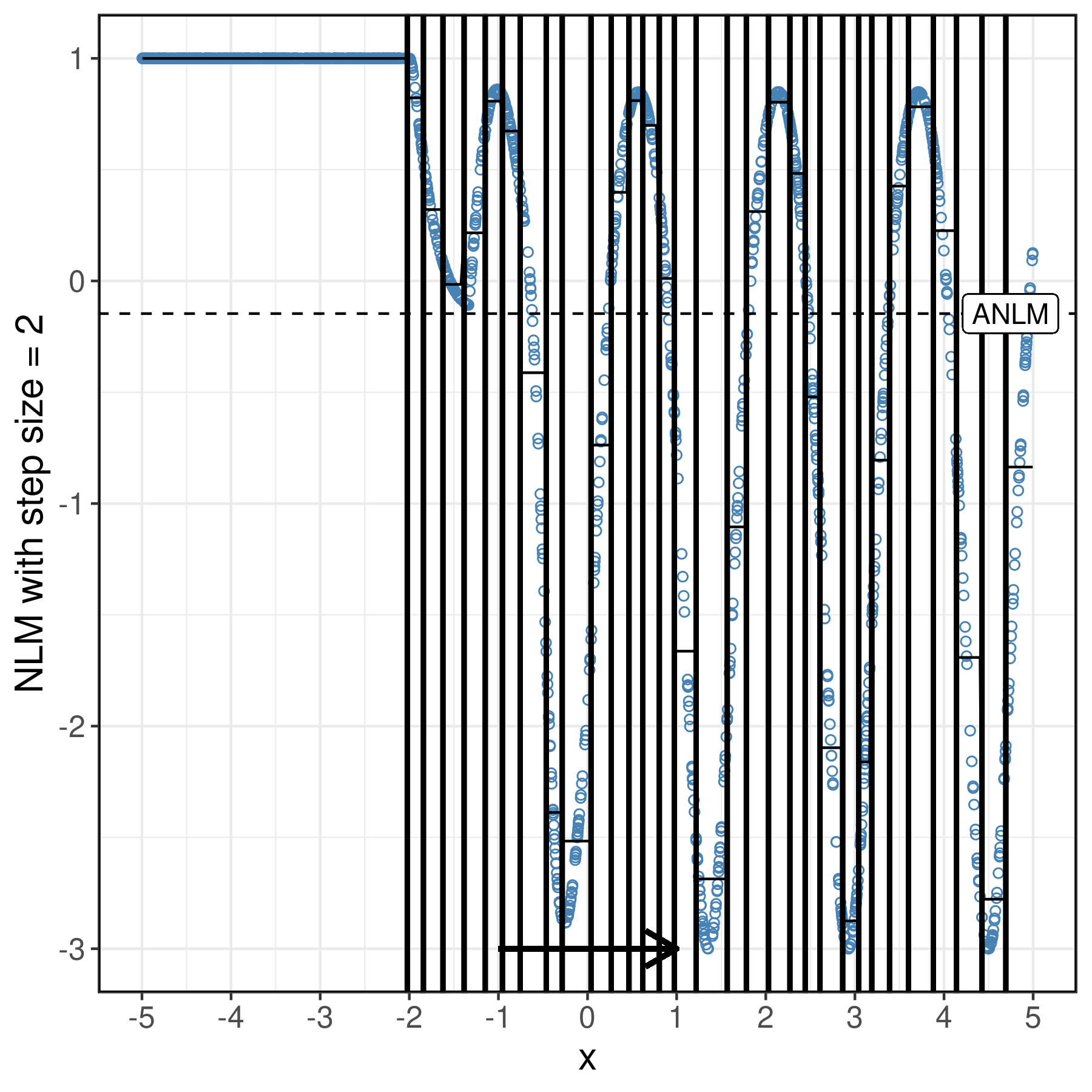}
\caption{\label{fig:simulation_univariate_no_noise_fme_nlm} \textbf{Univariate data without noise}. For each point, moving in x direction by the length of the arrow results in the fME / NLM indicated on the vertical axis.
\textbf{Left}: fMEs with step size $h = 2$. A regression tree partitions the feature space into subspaces (in this case intervals) where the fMEs are most homogeneous. The horizontal lines correspond to the cAMEs. 
\textbf{Right}: NLM values and cANLMs for each subspace.
}
\end{figure}

\subsection{Univariate Data with Noise}

We proceed to add random noise $\epsilon \sim N(0, 1)$ on top of the data and tune the regularization and sigma parameters of a support vector machine (SVM) with a radial basis function kernel (see Fig. \ref{fig:simulation_univariate_data}).
As we now employ a predictive model, we must avoid potential model extrapolations. The forward location of all points with $x > 3$ falls outside the range of the training data. After removing all extrapolation points, we evaluate the fMEs and NLMs of all observations with $x \in [-5, 3]$ (see Fig. \ref{fig:fME_univariate_sim}). In this case, we can visually assess that the predictions of the SVM resemble the DGP but also factor in noise (see Fig. \ref{fig:simulation_univariate_data}). e.g., the SVM prediction function is non-linear in linear regions of the DGP, which affects the fMEs and NLMs. This demonstrates that fMEs can only be used to explain the DGP if the model describes it accurately.

\begin{figure}
\centering
 \includegraphics[width=0.49\textwidth]{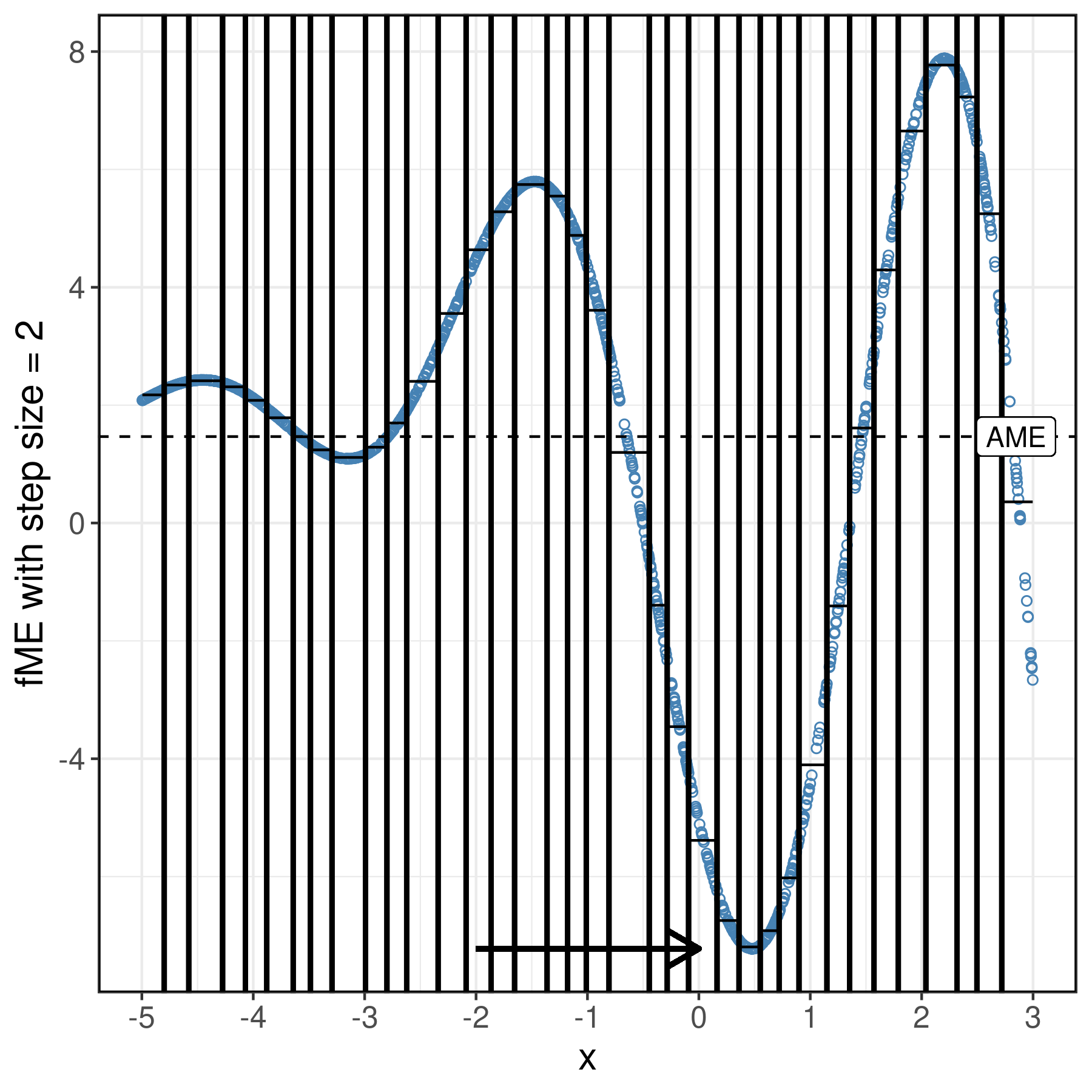}
 \includegraphics[width=0.49\textwidth]{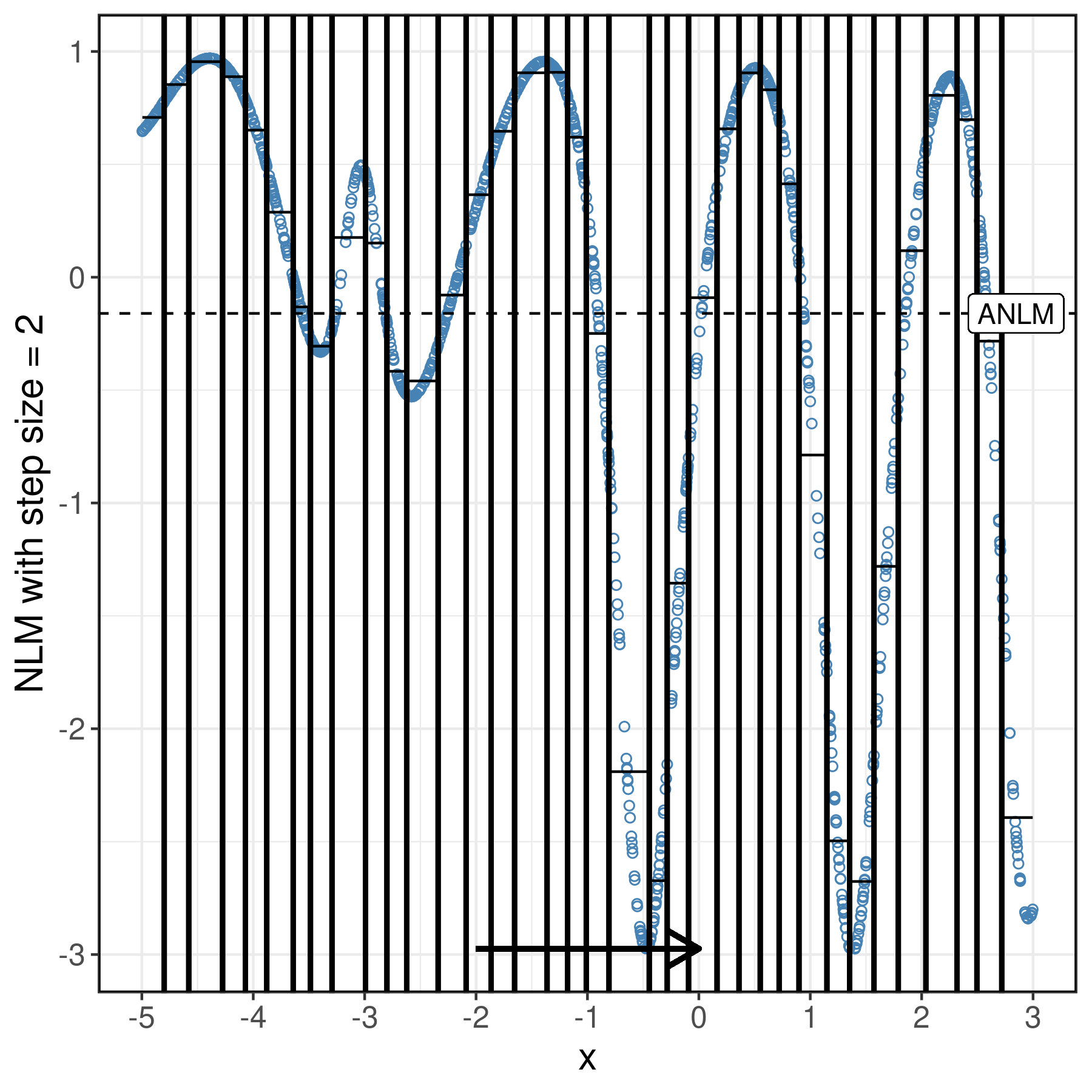}
\caption{\label{fig:fME_univariate_sim} \textbf{Univariate data with noise}. For each point, moving in x direction by the length of the arrow results in the fME / NLM indicated on the vertical axis.
\textbf{Left}: fMEs with step size $h = 2$ and cAMEs.
\textbf{Right}: NLM values and cANLMs.
}
\end{figure}

\subsection{Bivariate Data with Univariate Feature Change}

We next augment the univariate data with one additional feature in order to empirically evaluate the additive recovery property of the fME (see Appendix \ref{app:additiverecovery}). Due to potential model extrapolations, we only make use of the DGP. In the first example, the DGP corresponds to a supplementary additively linked feature $x_2$:
\begin{equation}
\label{eq:data_gen_process_bivariate_additive}
f(x_1, x_2) =
\begin{cases} x_1 + x_2 & x_1 < 0 \\ 5 \sin (2x_1) + x_2 & x_1 \geq 0
\end{cases}
\end{equation}
In the second example, the DGP corresponds to a supplementary multiplicatively linked feature $x_2$, i.e., we have a pure interaction:
\begin{equation}
\label{eq:data_gen_process_bivariate_multiplicative}
f(x_1, x_2) =
\begin{cases} x_1 \cdot x_2  & x_1 < 0 \\ 5 \sin (2x_1) \cdot x_2 & x_1 \geq 0
\end{cases}
\end{equation}
The fMEs and NLM values for both DGPs are given in Fig. \ref{fig:simulation_data_bivariate_additive_multiplicative}. For the additive DGP, given the value of $x_1$, moving in $x_2$ direction does not influence the fMEs due to the additive recovery property. As a result, we receive the same fMEs with an additively linked feature $x_2$ as without it (as long as the feature change does not occur in $x_2$). For the multiplicative DGP, the fMEs now vary for a given $x_1$ value, even though the feature change only occurs in $x_1$. The NLM values are both affected by the presence of an additively linked and a multiplicatively linked feature $x_2$, even though the feature change only occurs in $x_1$. As opposed to the additive DGP, the cAME tree makes use of $x_2$ as a split variable for the multiplicative DGP.

\begin{figure}
 \begin{subfigure}[t]{\textwidth}
 \includegraphics[width = 0.49\linewidth]{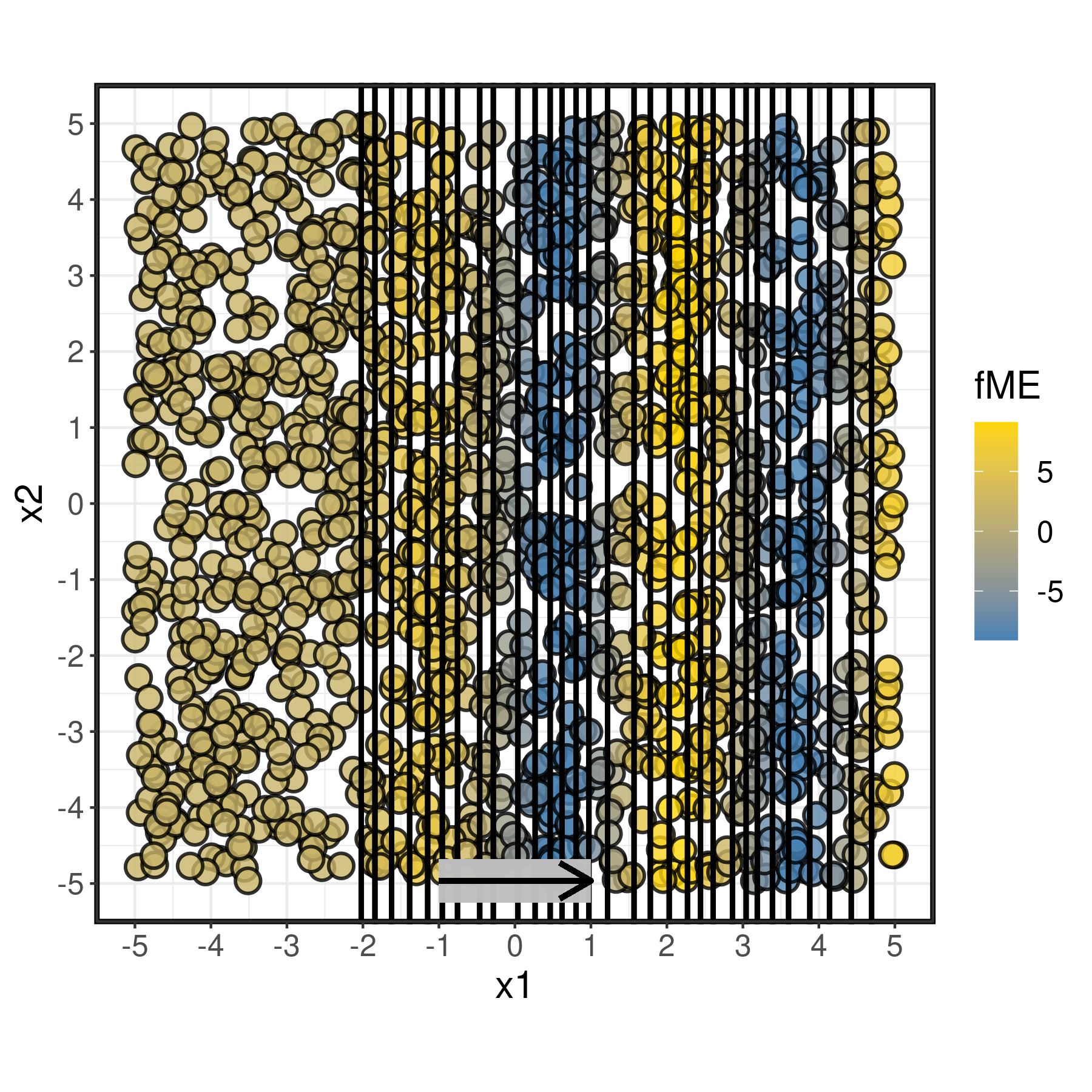}
\includegraphics[width = 0.49\linewidth]{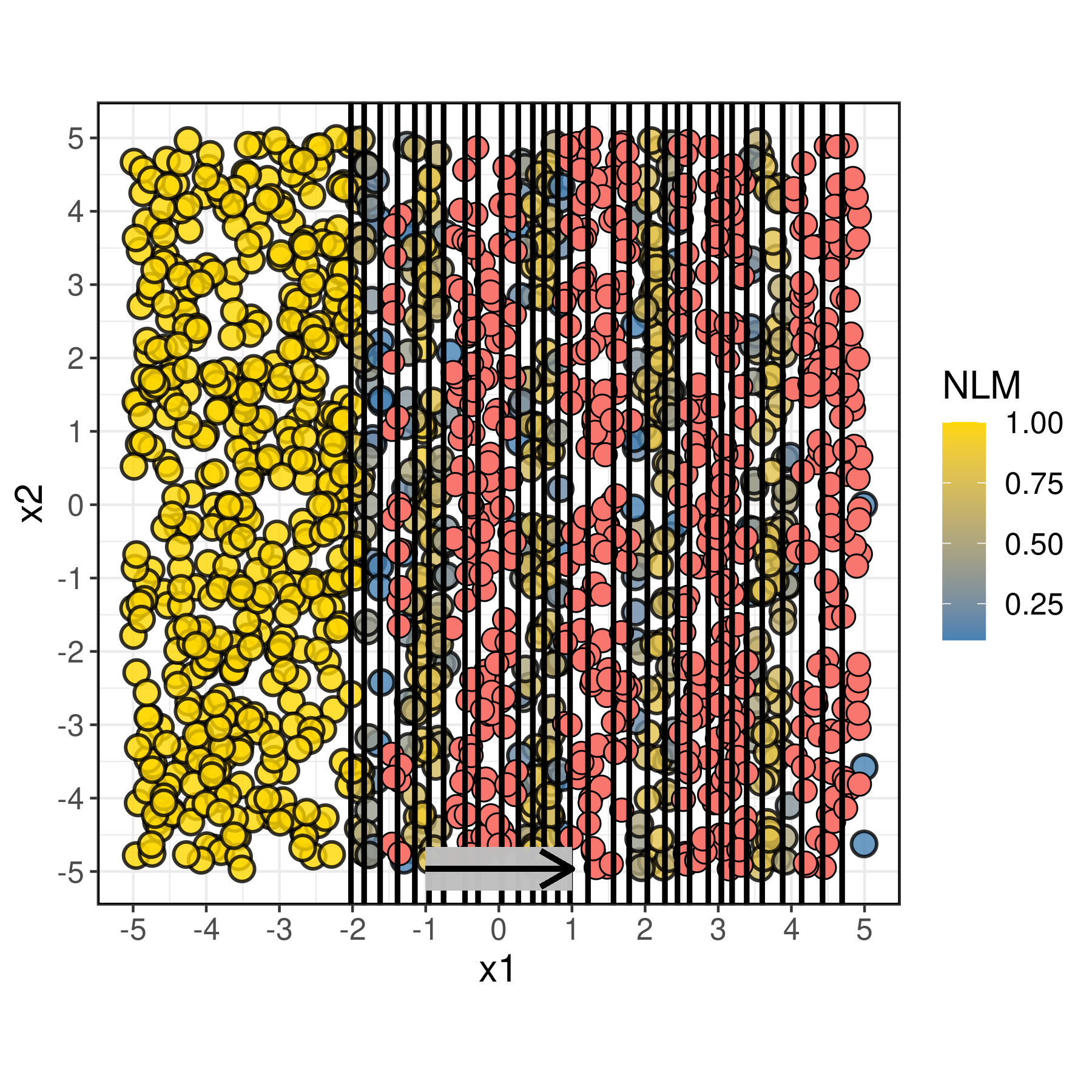}
\caption{DGP with additive link.}
\end{subfigure}
\begin{subfigure}[t]{\textwidth}
\includegraphics[width = 0.49\linewidth]{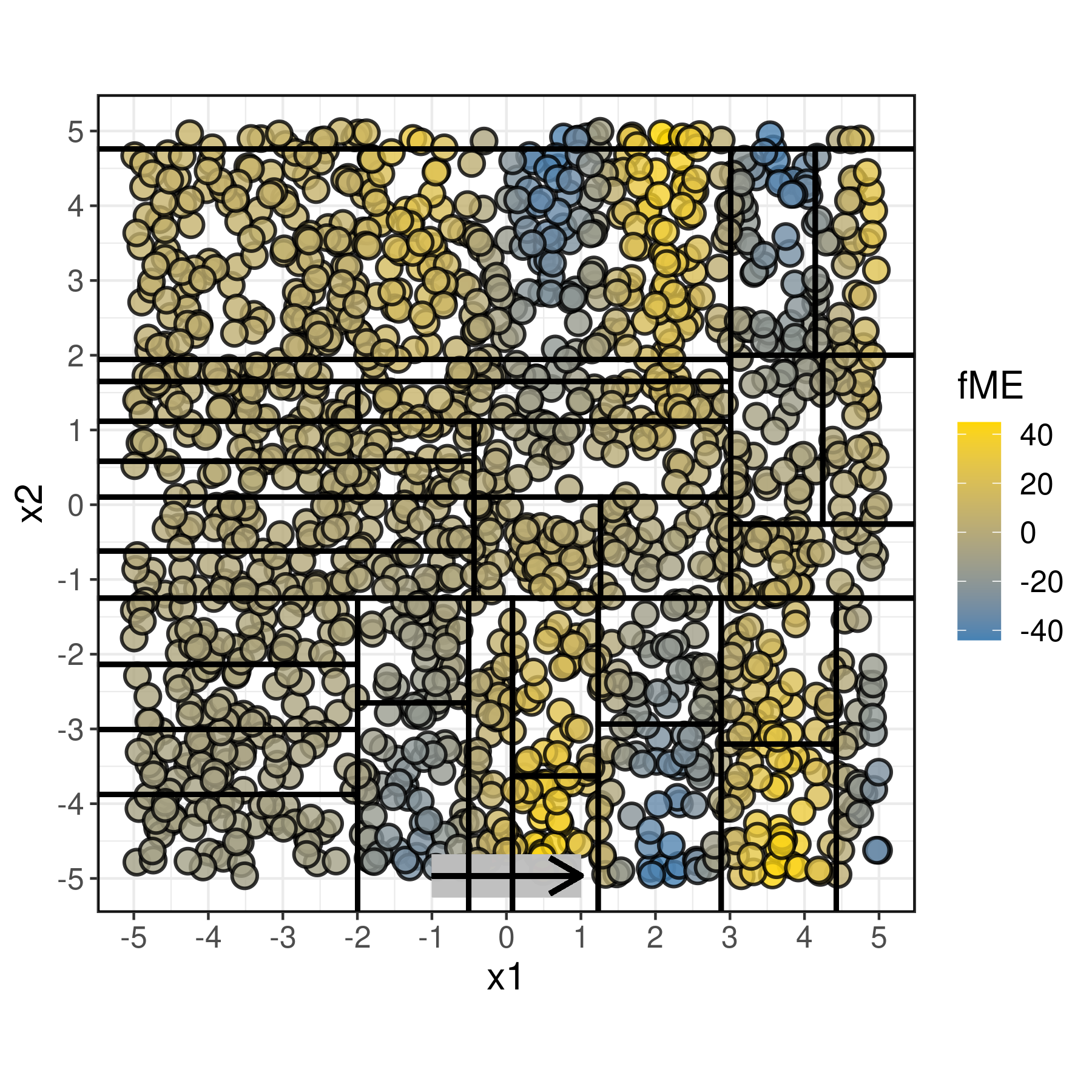}
\includegraphics[width = 0.49\linewidth]{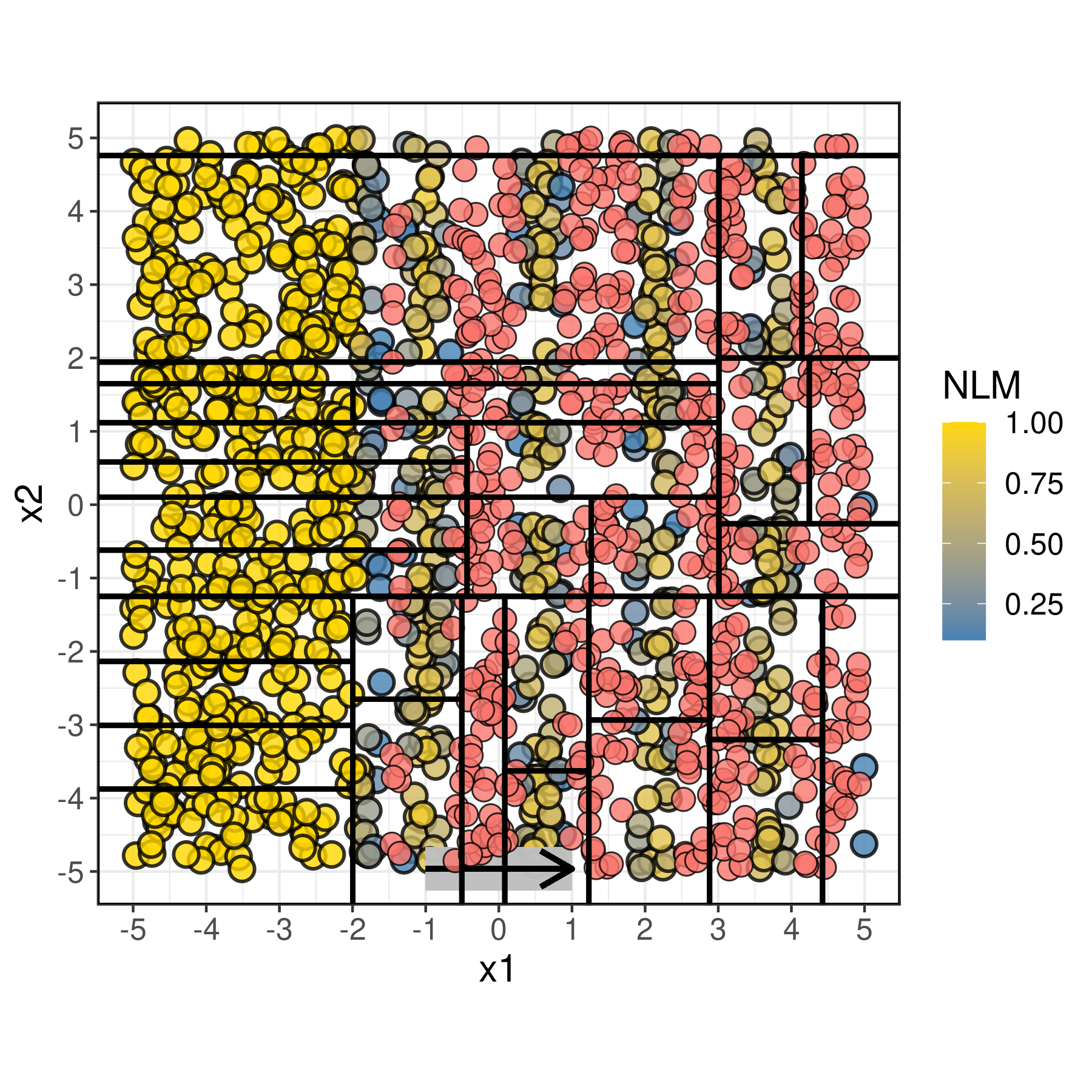}
\caption{DGP with multiplicative link.}
\end{subfigure}
\caption{\label{fig:simulation_data_bivariate_additive_multiplicative} \textbf{Bivariate data and univariate feature change $\mathbf{h_1 = 2}$.} For each point, moving in $x_1$ direction by the length of the arrow results in the fME / NLM indicated by the color. fMEs (left) and NLM (right). Negative NLM values are red colored.}
\end{figure}

\subsection{Bivariate Data with Bivariate Feature Change}

In the last simulation example, we demonstrate bivariate fMEs and the corresponding NLM. We use the same DGPs as for the univariate feature change.
The fMEs and NLM values are given in Fig. \ref{fig:simulation_fme_nlm_bivariate_bivariate_change}.
As opposed to the univariate feature change for additively linked data, the fME values now also vary in $x_2$ direction for a given $x_1$ value due to the simultaneous change in $x_2$. The NLM indicates linearity for a multitude of observations, given both the additive and the multiplicative DGP. For these observations, we can infer that multiplying \textit{both} step sizes by a value on the interval $[0, 1]$ results in an equally proportional reduced fME. 
 
\begin{figure}
 \begin{subfigure}{\textwidth}
 \includegraphics[width = 0.49\linewidth]{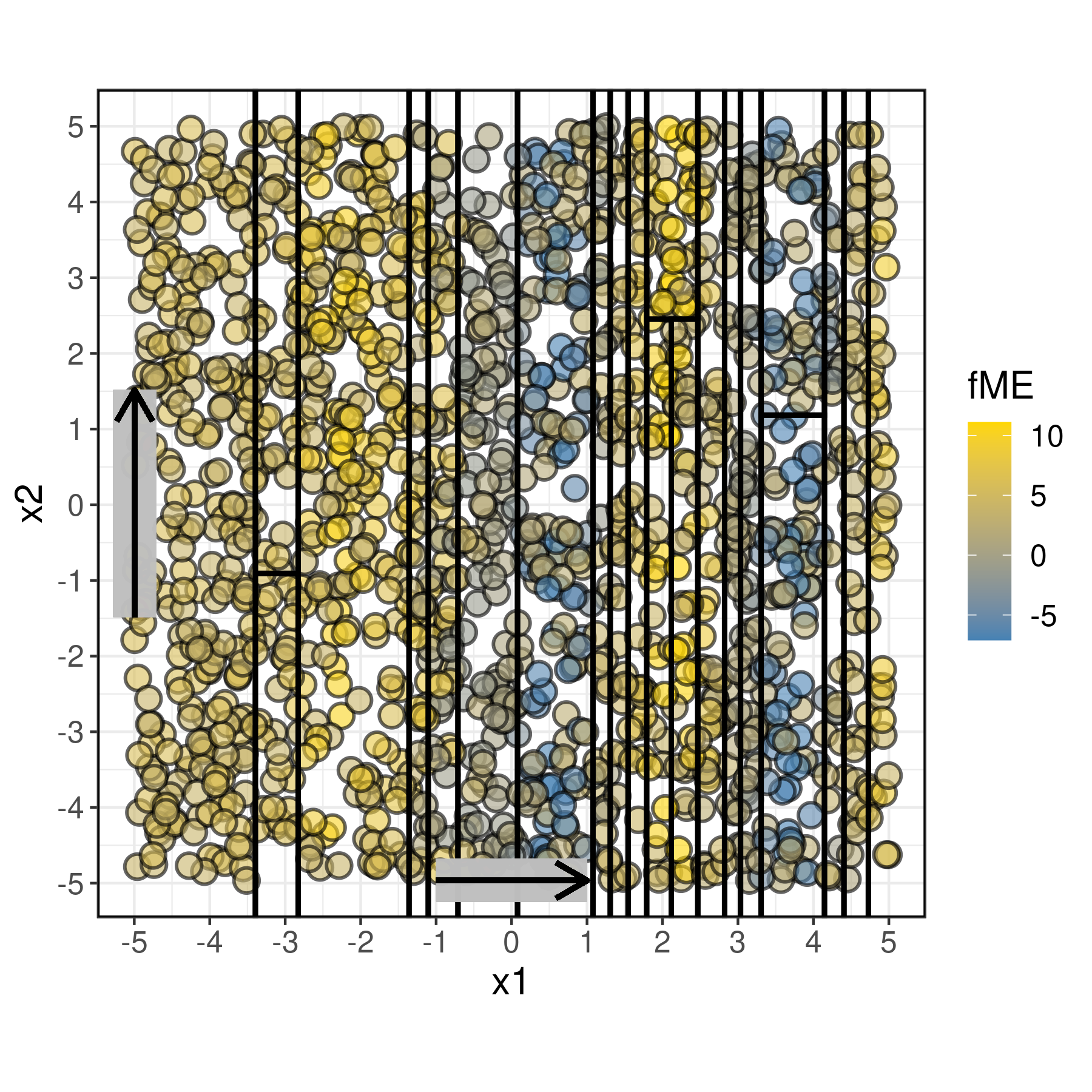}
\includegraphics[width = 0.49\linewidth]{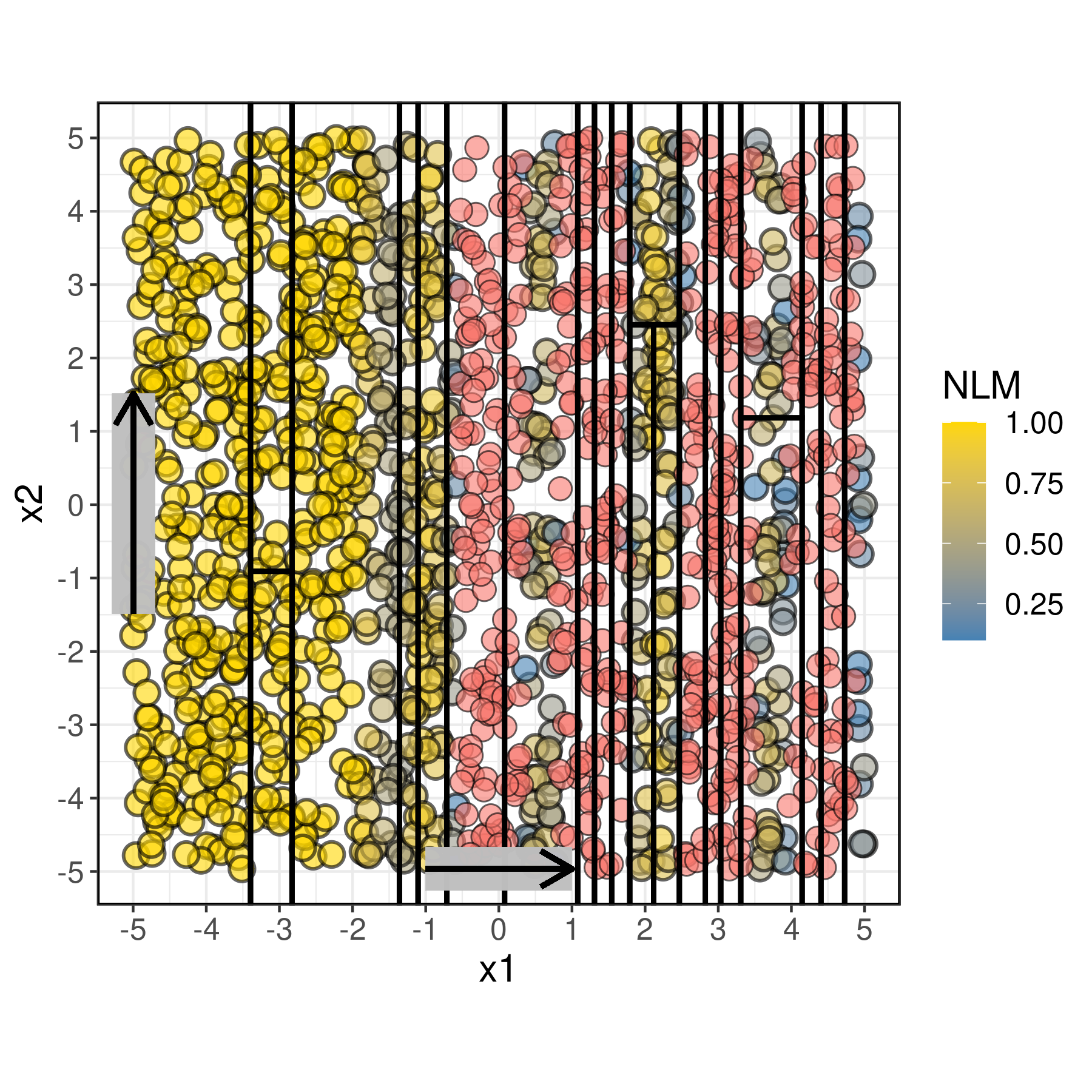}
\caption{DGP with additive link.}
\end{subfigure}

\begin{subfigure}{\textwidth}
\includegraphics[width = 0.49\linewidth]{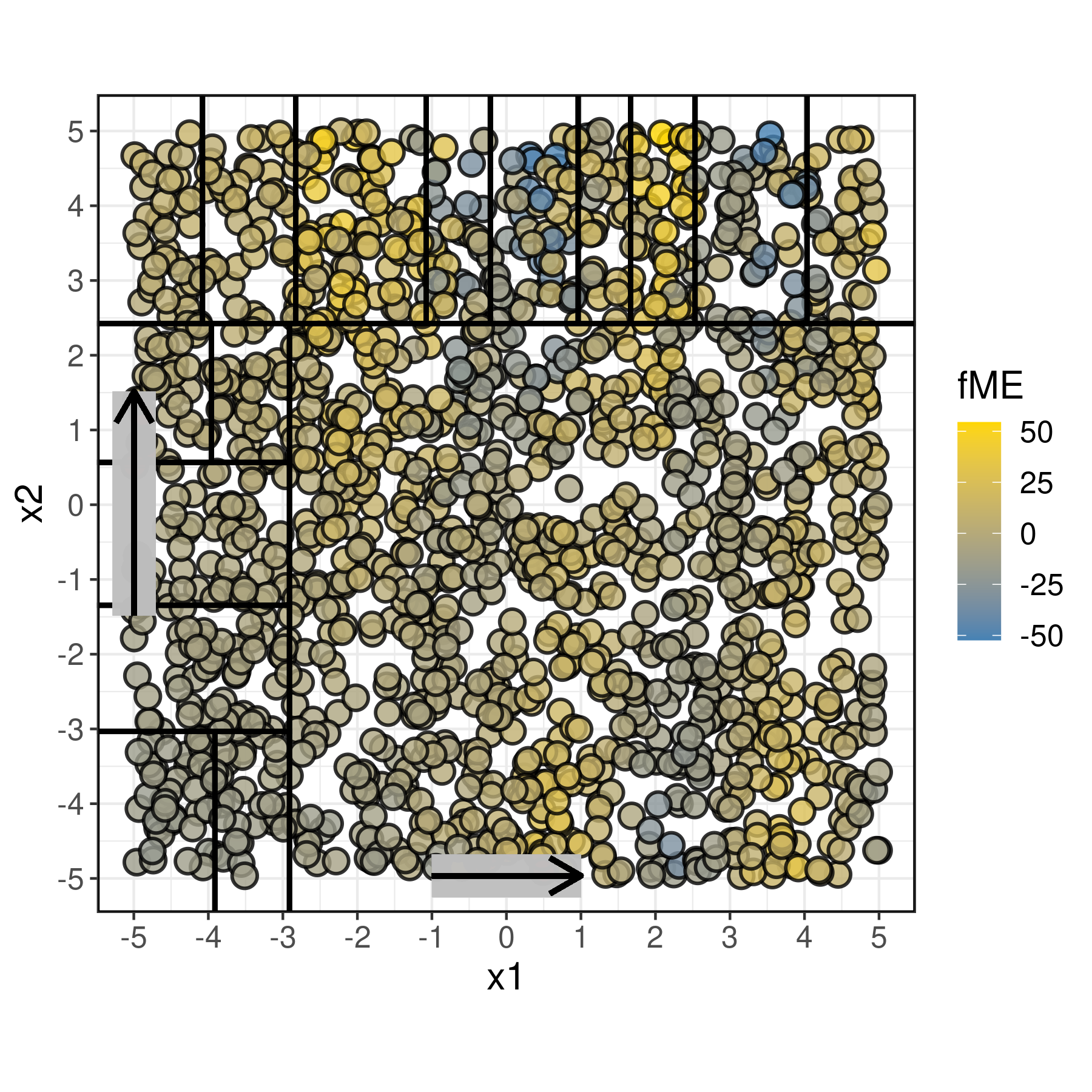}
\includegraphics[width = 0.49\linewidth]{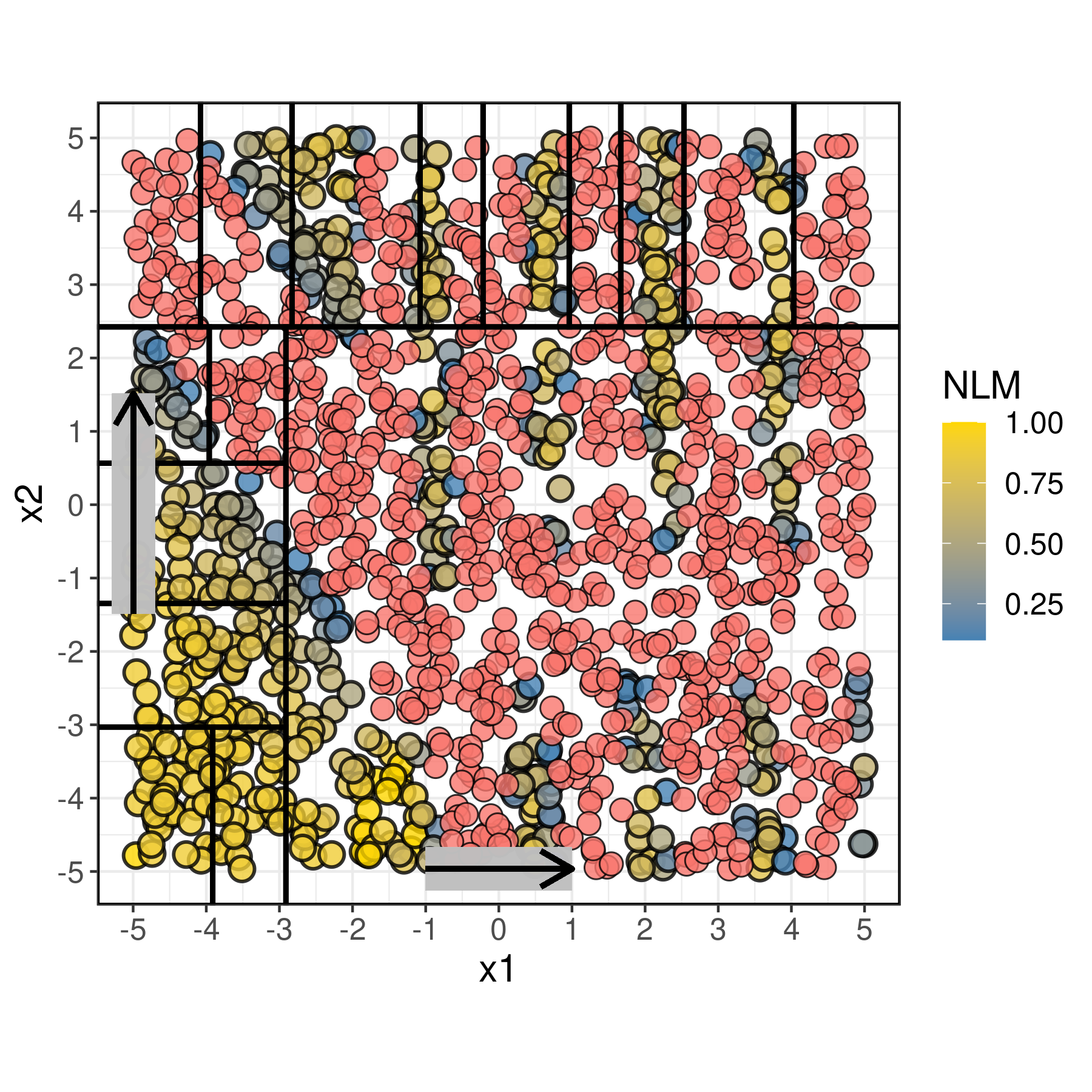}
\caption{DGP with multiplicative link.}
\end{subfigure}
\caption{\label{fig:simulation_fme_nlm_bivariate_bivariate_change}\textbf{Bivariate data and bivariate feature change $\mathbf{h_1 = 2}$ and $\mathbf{h_2 = 3}$}. For each point, moving in $x_1$ and $x_2$ directions by the lengths of the respective arrows results in the fME / NLM indicated by the color. fMEs (left) and NLM (right). Negative NLM values are red colored.}
\end{figure}

\section{Application Workflow and Applied Example}
\label{sec:workflow_application}

\begin{figure}
\caption{\label{fig:workflow}Application workflow for forward marginal effects.}
\begin{enumerate}
    \item Train and tune a predictive model.
    \item Based on the application context, choose evaluation points $\mathcal{D}$, the features of interest $S$, and the step sizes $\boldhS$.
    \item Check whether any $\boldx^{(i)}$ or $(\boldxS^{(i)} + \boldhS, \boldxMinusS^{(i)})$ are subject to model extrapolations. See Appendix \ref{app:extrapolation} for possible options.
    \item Either modify the step sizes so no points fall inside areas of extrapolation or remove the ones that do.
    \item Compute fMEs for  selected observations and the chosen step sizes.
    \item Compute the NLM for every computed fME.
    \item Compute cAMEs by finding feature subspaces with homogeneous fME values.
    \item Compute cANLM values.
    \item Optional: Compute CIs for cAME and cANLM.
    \item Conduct local (single fMEs of interest) and semi-global interpretations (cAME and cANLM).
\end{enumerate}
\end{figure}
\begingroup
\renewcommand*{\arraystretch}{1.1}
\begin{table}
\caption{\label{tab:boston_housing_data}Description of the Boston housing data.}
\scriptsize
\begin{tabularx}{\linewidth}{|l|X|l|}
\hline
Variable & Description & Range \\
 \hline
 crim & Per capita crime rate by town. & [0.00632, 88.9762]\\
 zn & Proportion of residential land zoned for lots over 25,000 sq.ft.  & [0, 100] \\
 indus & Proportion of non-retail business acres per town. & [0.46, 27.74]  \\
 chas & Charles River dummy variable (= 1 if tract bounds river; 0 otherwise).  & \{0, 1\} \\
 nox & Nitrogen oxides concentration (parts per 10 million). & [0.385, 0.871] \\
 rm & Average number of rooms per dwelling.  & [3.561, 8.78] \\
 age & Proportion of owner-occupied units built prior to 1940. & [2.9, 100] \\
 dis & Weighted mean of distances to five Boston employment centres. & [1.1296, 12.1265]\\
rad & Index of accessibility to radial highways. & [1, 24] \\
 tax & Full-value property-tax rate per USD 10,000. & [187, 711] \\
 afram & $1000(AA - 0.63)^2$ where $AA$ is the prop. of African Americans by town & [0.32, 396.90]\\
 ptratio & Pupil-teacher ratio by town.  & [12.6, 22]\\
 lstat & Lower status of the population (percent). & [1.73, 37.97] \\
 medv & Median value of owner-occupied homes in USD 1000s (censored above 50,000 USD).  & [5, 50]\\
 \hline
\end{tabularx}
\normalsize
\end{table}
\endgroup

We first present a structured application workflow (see Fig. \ref{fig:workflow}) that incorporates all the theory presented in the preceding sections and demonstrate it with the Boston housing data (see Table \ref{tab:boston_housing_data} for a description). 
We start by tuning the regularization and sigma parameters of an SVM with a radial basis function kernel and proceed to evaluate feature effects on the predicted median housing value per census tract in Boston. The housing prices reflect the market demand, which in turn represents the attractiveness of living in the city or a certain district. For instance, the local authorities might be interested in the effects of lowering the crime rate by 5\% (e.g., due to a larger police force) on the housing market. In order to avoid model extrapolations, we exclude all points that are located outside of the multivariate envelope of the training data. We start with the first order effect (see Fig. \ref{fig:boston_housing_fme_nlm_crim}). For the majority of observations, decreasing the crime rate by 5\% considerably increases the predicted median housing value by several thousand US dollars (USD). For a fraction of observations, the predicted change in median housing value is close to zero or below zero. The effect is approximately linear for all observations.
\par
Next, we evaluate potential interactions. In Fig. \ref{fig:boston_housing_fme_nlm_crim_nox}, we visualize the interaction between the features \texttt{crim} and \texttt{nox} with the step sizes $h_{\text{crim}} = -5$ and $h_{\text{nox}} = -0.1$, i.e., a decrease in the crime rate by 5\% and a simultaneous decrease in the nitrogen oxides concentration by 0.1 parts per 10 million (range between 0.385 and 0.871). For the majority of observations, there is a considerable increase in the median housing value, which is larger than for the isolated decrease in the crime rate. We proceed to three-way interactions. In Fig. \ref{fig:boston_housing_fme_nlm_crim_nox_ptratio}, we visualize the kernel density estimate of the distribution of fMEs and NLM values for the step sizes $h_{\text{crim}} = -5$, $h_{\text{nox}} = -0.1$ and $h_{\text{ptratio}} = -1$. For a simultaneous threefold decrease of the given magnitudes, there is a large increase in predicted median housing value for the majority of observations. The threefold interaction demonstrates an advantage of fMEs over competing techniques such as ICEs or the PD, which are difficult to interpret in more than two dimensions. Most three-way fMEs have a high degree of linearity with an NLM close to 1.
\par
We proceed to evaluate cAMEs with the step size $h_{\text{crim}} = -5$. In Fig. \ref{fig:boston_housing_cAME_tree}, we recursively partition the feature space with CTREE to find feature subspaces with a high fME homogeneity. In this example, the fMEs on the feature space are already rather homogeneous to begin with (CoV = 0.32), as can be seen by the large cluster of fMEs in Fig \ref{fig:boston_housing_fme_nlm_crim}. The minimum node size is set to 30 so we can compute CIs. We succeed in finding two subspaces with a lower fME CoV value than the root node, except for one region with a higher CoV value. This demonstrates that, although the aggregate impurity of the fMEs can be reduced via recursive partitioning, we may still find subspaces with a larger impurity than the root node.
Lowering the crime rate by 5\% has the smallest effect in districts with high crime rates. For lower crime rates, the feature effect is slightly larger in areas with a large nitrogen oxides concentration.
\par
In a final step, we demonstrate the construction of CIs. For an observation that falls into the rightmost leaf node, the estimated cME is 0.86 (given by the cAME), and the estimated cNLM is 0.99 (given by the cANLM). The 95\% CIs correspond to $\text{CI}_{\text{cAME}, \; 0.95} = [0.65, 1.07]$ and $\text{CI}_{\text{cANLM}, \; 0.95} = [0.98, 1]$.

\begin{figure}
\centering
\begin{subfigure}[t]{0.8\textwidth}
\includegraphics[width = 0.49\linewidth]{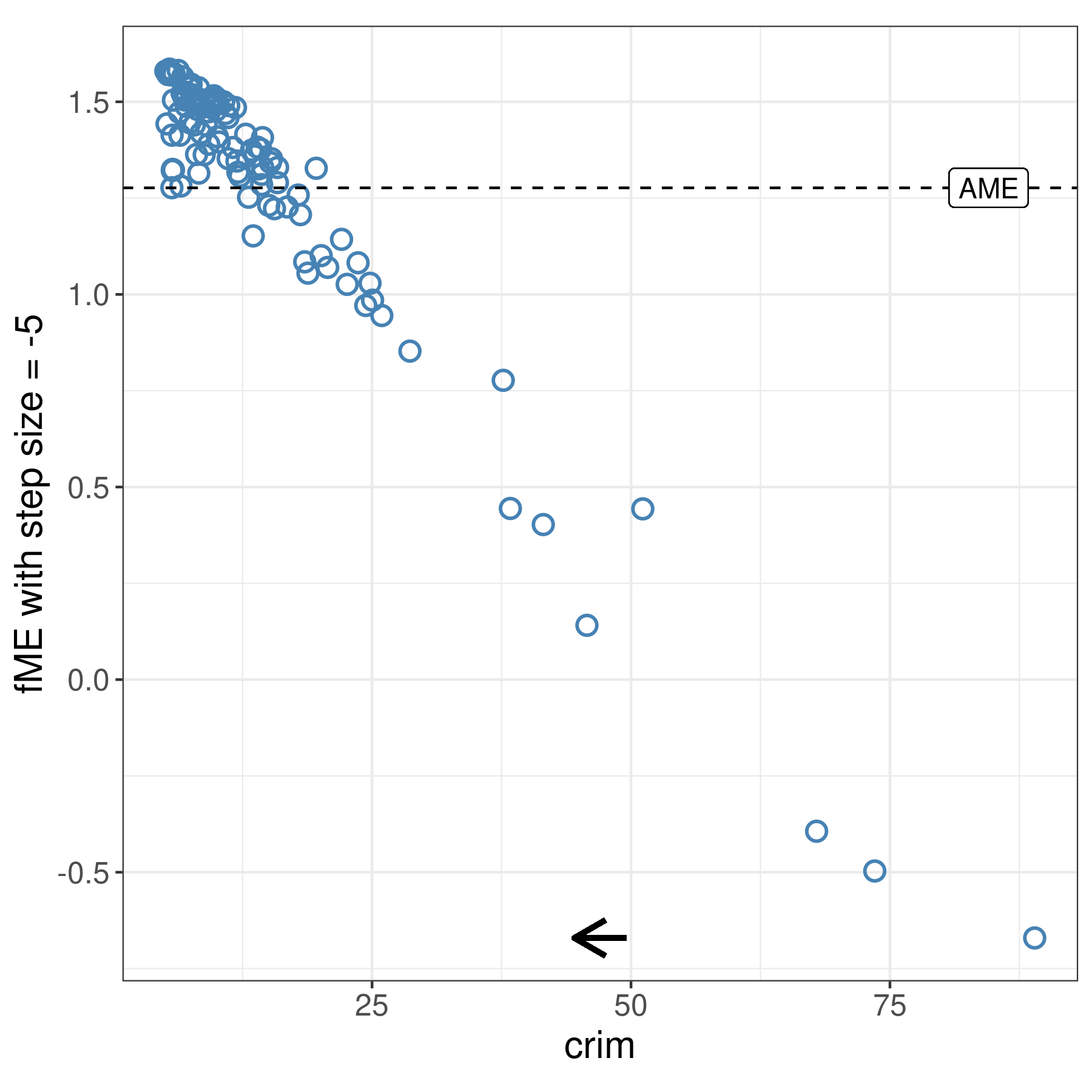}
\includegraphics[width = 0.49\linewidth]{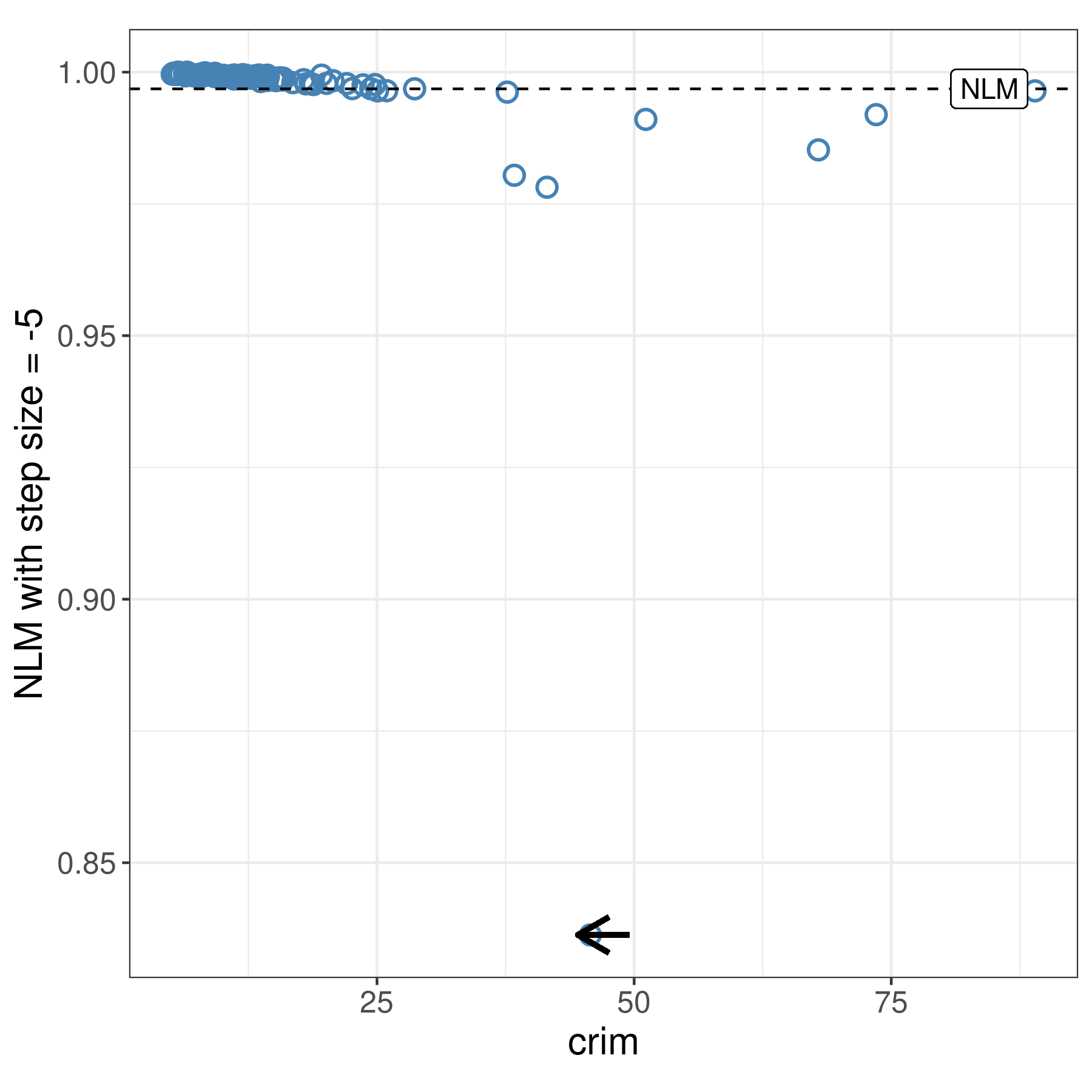}
\caption{\label{fig:boston_housing_fme_nlm_crim}fMEs (left) and NLM (right) for the feature \texttt{crim} and step size $h_{\text{crim}} = -5$.}
\end{subfigure}
\begin{subfigure}[t]{0.8\textwidth}
 \includegraphics[width = 0.49\linewidth]{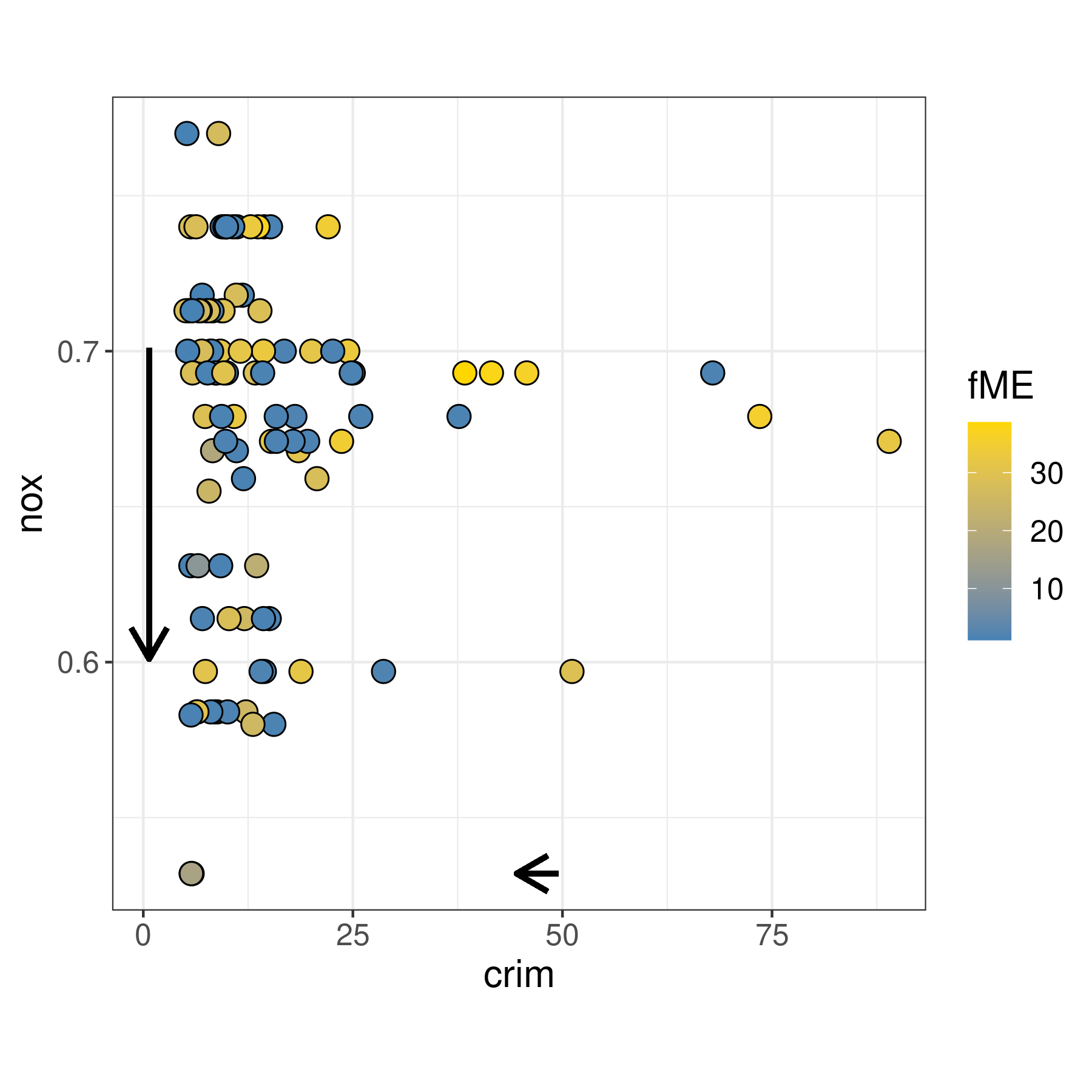}
\includegraphics[width = 0.49\linewidth]{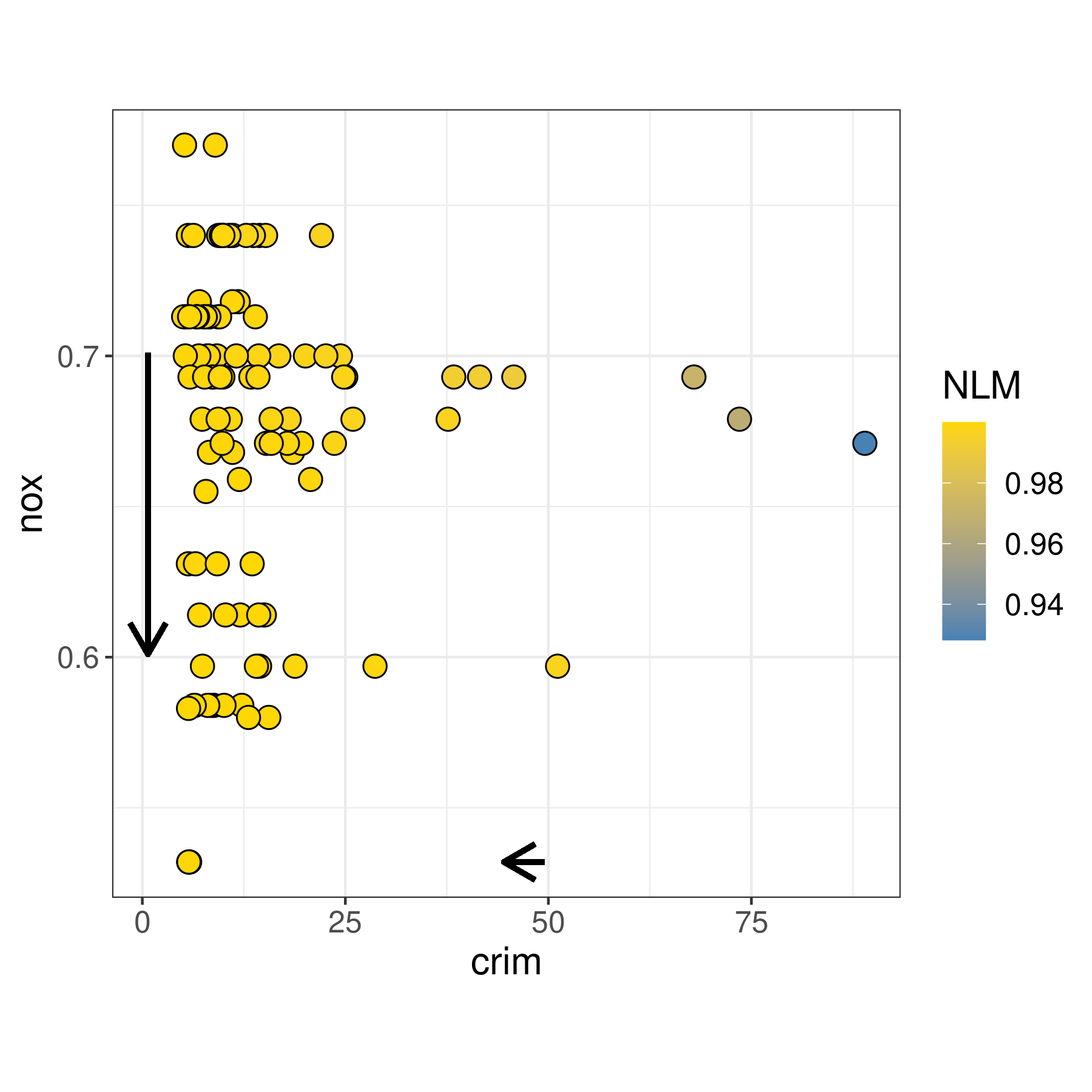}
\caption{\label{fig:boston_housing_fme_nlm_crim_nox}fMEs (left) and NLM (right) for the features \texttt{crim} and \texttt{nox}, and step sizes $h_{\text{crim}} = -5$ and $h_{\text{nox}} = -0.1$.}
\end{subfigure}
\begin{subfigure}[t]{0.8\textwidth}
\includegraphics[width = 0.49\linewidth]{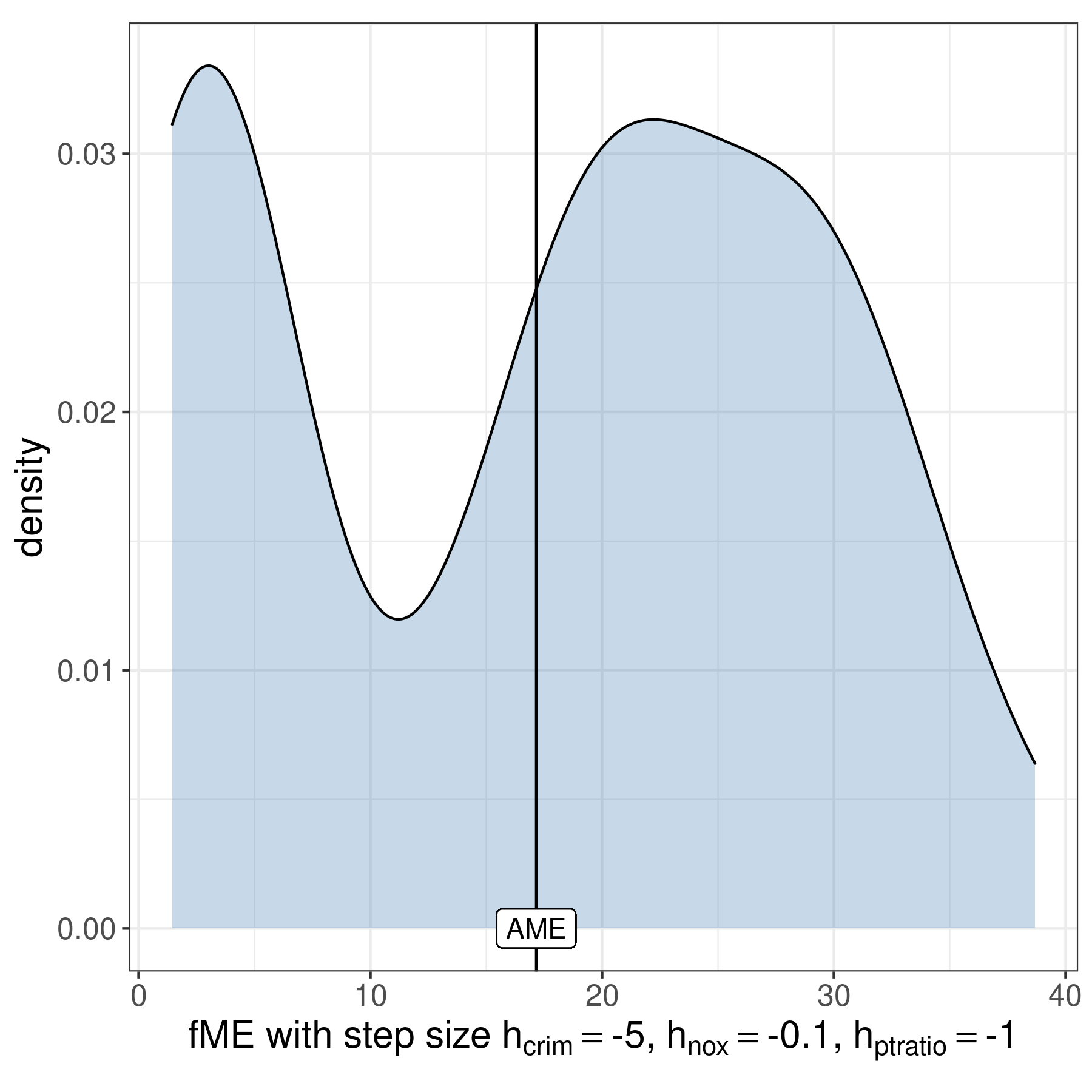}
\includegraphics[width = 0.49\linewidth]{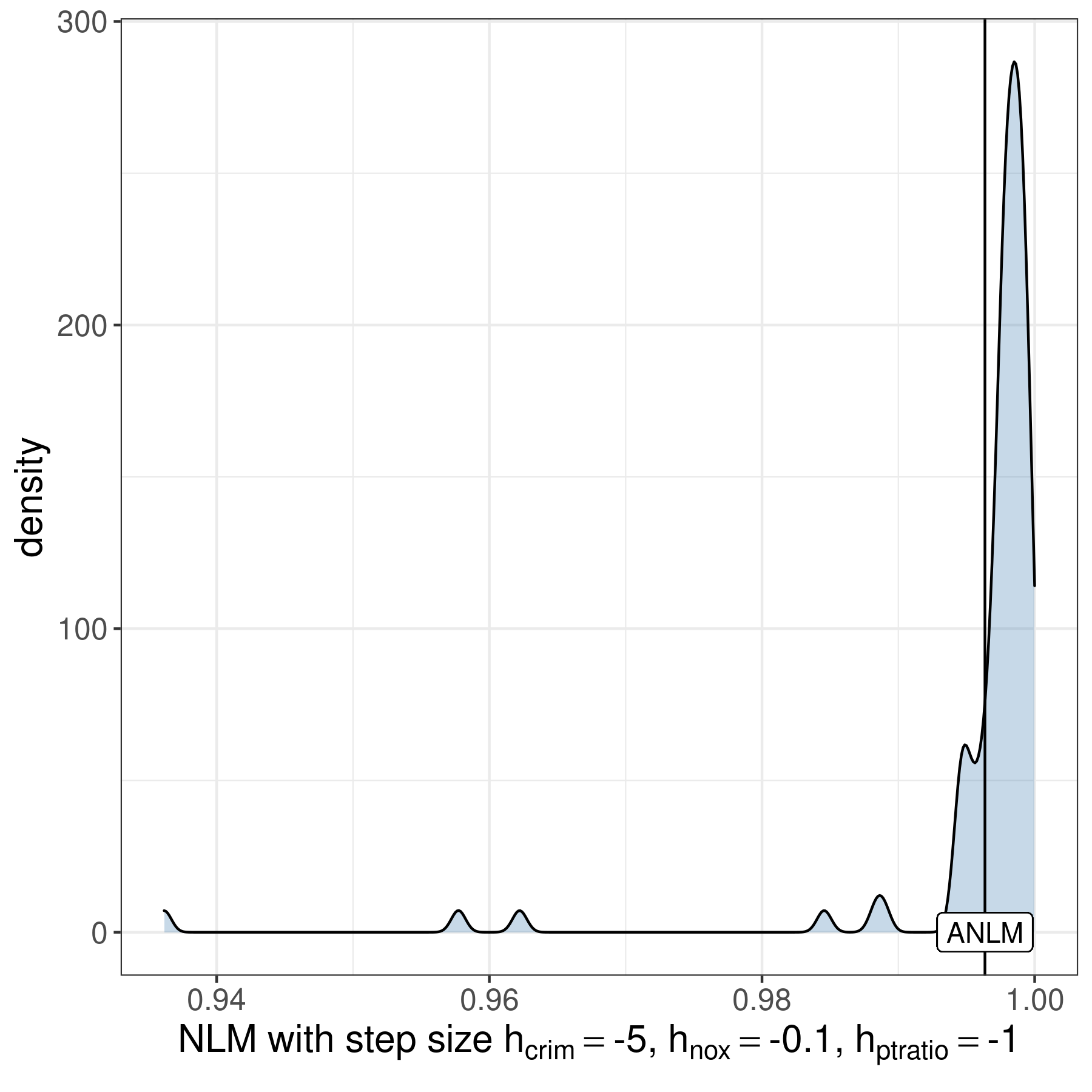}
 \caption{\label{fig:boston_housing_fme_nlm_crim_nox_ptratio}fMEs (left) and NLM (right) for the features \texttt{crim}, \texttt{nox}, and \texttt{ptratio}, and step sizes $h_{\text{crim}} = -5$, $h_{\text{nox}} = -0.1$ and $h_{\text{ptratio}} = -1$.}
\end{subfigure}
\caption{Applied example with Boston housing data.}
\end{figure}

\begin{figure}
\centering
 \includegraphics[width=\linewidth]{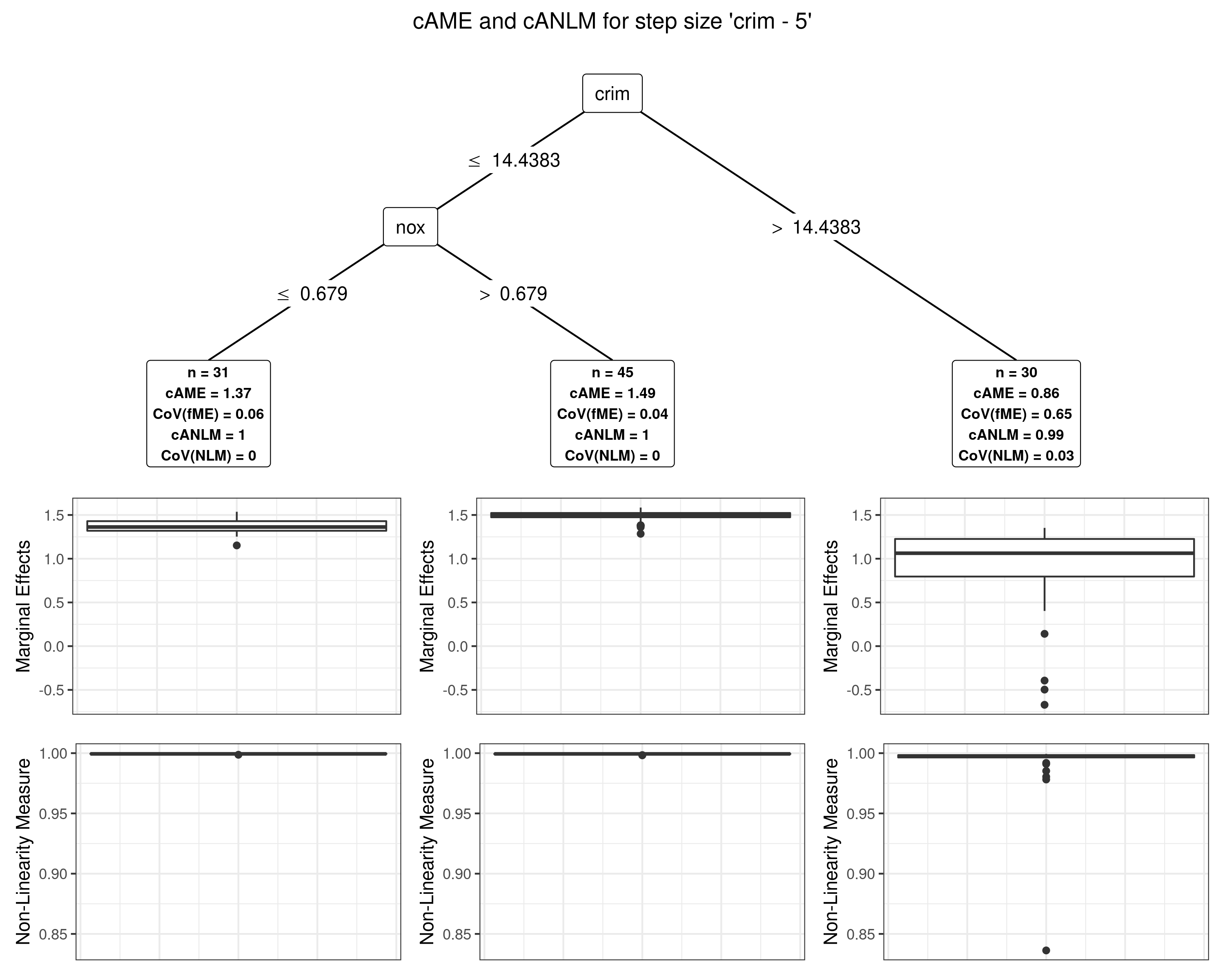}
\caption{\label{fig:boston_housing_cAME_tree}cAME and cANLM values for the step size $h_{\text{crim}} = -5$ on feature subspaces created via CTREE.}
\end{figure}

\section{Conclusion}

This research paper intends to both provide a comprehensive review of the theory of MEs and to refine it for the application to non-linear prediction functions, e.g., in the context of ML applications. We argue to abandon dMEs in favor of using more accurate fMEs. To estimate multivariate feature effects, we extend the existing definition of univariate changes in feature values to multiple dimensions. Furthermore, we introduce an NLM for fMEs, which indicates the similarity between the prediction function and the intersecting linear secant. Due to the complexity and non-linearity of ML models, we suggest to focus on semi-global instead of global feature effects. We propose a means of estimating cMEs via cAMEs through partitioning the feature space. The resulting subspaces are augmented with cANLM values and CIs in order to receive a compact summary of the prediction function across feature subspaces. Furthermore, we propose an observation-wise categorical ME. In the Appendix, we provide proofs on the additive recovery property of fMEs and their relation to the ICE and PD.
\par
Given arbitrary predictive models, this new class of of fMEs can be used to address questions on a model's behavior such as the following: Given pre-specified changes in one or multiple feature values, what is the expected change in predicted outcome? What is the change in prediction for an average observation? What is the change in prediction for a pre-specified observation? What are population subgroups with homogeneous expected effects? What is the degree of non-linearity in these effects? What is our confidence in these estimates? What is the expected change in prediction when switching observed categorical feature values to a reference category?
\par
However, there are certain limitations as well. \citet{molnar_pitfalls} discuss various general pitfalls of model-agnostic interpretation methods, e.g., model extrapolations, estimation uncertainty, or unjustified causal interpretations. Model-agnostic methods in general are favorable tools to explain the model behavior but often fail to explain the underlying DGP, as the quality of the explanations relies on the closeness between model and reality. 
\par
Throughout the manuscript, we noted various directions that may be explored in future work. These include quantifying extrapolation risk for the selection of step sizes, stabilizing the found subspaces for cAMEs, or quantifying the subspaces' uncertainty. The motivation behind an interpretation fundamentally determines whether a method is used as an effect or importance method. An effect can also be considered an importance, i.e., features with larger effects on the predicted outcome are considered more important to the model. We formulated fMEs as a method to determine feature effects. However, finite differences are already widely used for importance computations in SA (see Section \ref{sec:notation_background}), e.g., the elementary effects / Morris method, or derivative-based global sensitivity measures. One could think of ways to utilize fMEs for importance computations as well. For instance, although we formulated the cME conditional on a feature subspace as a semi-global effect, it can also be interpreted as a semi-global feature importance.
\par
Many disciplines that have been relying on traditional statistical models are starting to utilize the predictive power of modern ML models. These disciplines are used to interpretations in terms of MEs, the AME, MEM, or MER. With this research paper, we wish to bridge the gap between the established and restricted method of using MEs with traditional statistical models and the more flexible and capable approach of interpreting modern ML models with this new class of fMEs.

\newpage
\appendix
\section{Background Information}
\label{app:background}
\subsection{Decomposition of the Prediction Function}
\label{app:decomposition}
The prediction function to be analyzed may be very complex or even a black box. However, there are multiple ways to decompose the prediction function into a sum of components of increasing order. Although the goal of MEs is not to decompose the prediction function, it is convenient to either regard the prediction function as an additive decomposition or to keep in mind that it may be decomposed into one. An additive decomposition of the prediction function has the following general form:
\begin{equation}
\widehat{f}(\boldX) = g_{\{0\}} + g_{\{1\}}(X_1) + g_{\{2\}}(X_2) + \dots + g_{\{1, 2\}}(X_1, X_2) + \dots + g_{\{1, \ldots, p\}}(\boldX) \label{eq:fanova}
\end{equation}
In SA, the additive decomposition is typically referred to as a high-dimensional model representation (HDMR) or ANOVA-HDMR \citep{saltelli_sa}. In IML, it often is called the functional ANOVA decomposition \citep{hooker_fanova}.
Various approaches exist to estimate Eq. (\ref{eq:fanova}) or a truncated variant, e.g., the PD \citep{friedman_pdp}, random sampling HDMR \citep{li_rs_hdmr_2006}, or accumulated local effects \citep{apley_ale}. Further assumptions are needed to make the decomposition unique, e.g., feature independence \citep{chastaing_hdmr_dependent}.
\par
For instance, we may recursively estimate Eq. (\ref{eq:fanova}) as follows:
\begin{align*}
 g_{\{0\}} &= \mathbb{E}_{\boldX}\left[\widehat{f}(\boldX)\right] \\
 g_{\{1\}}(X_1) &= \mathbb{E}_{X_{-1}}\left[\widehat{f}(\boldX) \; \vert  \; X_1 \right] - g_{\{0\}} \\
 g_{\{2\}}(X_2) &= \mathbb{E}_{X_{-2}}\left[\widehat{f}(\boldX) \; \vert  \; X_2 \right] - g_{\{0\}} \\
 g_{\{1, 2\}}(X_1, X_2) &= \mathbb{E}_{\bm{X}_{-\{1,2\}}}\left[\widehat{f}(\boldX) \; \vert \; X_1, X_2 \right] - g_{\{2\}}(X_2) - g_{\{1\}}(X_1) - g_{\{0\}}\\
 &\vdots \\
 g_{\{1, \dots, p\}}(\boldX) &= \widehat{f}(\boldX) - \dots - g_{\{1, 2\}}(X_1, X_2) - g_{\{2\}}(X_2) - g_{\{1\}}(X_1) - g_{\{0\}}\\
\end{align*}
where $\mathbb{E}_{\bm{X}_{-S}}\left[\widehat{f}(\boldX) \; \vert \; \boldX_S \right]$ corresponds to the minimum L2 loss approximation of $\widehat{f}(\boldX)$ given only the features $\boldX_S$ \citep{saltelli_sa} and is typically referred to as the PD in ML terms.

\subsection{Model Extrapolation}
\label{app:extrapolation}

\citet{king_extrapolation} define extrapolation as predicting outside of the convex hull of the training data. The authors demonstrate that the task of determining whether a point is located inside of the convex hull can be efficiently solved using linear programming. However, the convex hull may be comprised of many empty areas without training observations, especially in the case of correlated and high-dimensional data. Therefore, it seems plausible to define model extrapolation differently, e.g., as predictions in areas of the feature space with a low density of training points. \citet{hooker_CERT} summarizes two main predicaments of model extrapolations. First, the model creates predictions which do not accurately reflect the target distribution given the features. Second, the predictions are subject to a high variance. Many model-agnostic techniques are subject to model extrapolation risks \citep{molnar_pitfalls}. \citet{hooker_generalizedfanova} warns against model extrapolations when applying the functional ANOVA decomposition. \citet{hooker_importance} urge not to permute feature values to compute feature importance values. It is important to note that this issue highly depends on the behavior of the chosen model. The issue of determining whether the model extrapolates essentially boils down to quantifying the prediction uncertainty. Some models might diverge considerably from a scenario where they would have been supplied with enough training data (high prediction uncertainty), while other models might be relatively robust against such issues (low prediction uncertainty). 
Although MEs based on model extrapolations are still correct in terms of the model output, they might not represent any underlying DGP in an accurate way. Therefore, it is important to take into account (and preferably avoid) potential model extrapolations when selecting feature values to compute fMEs.
\par
For some models, built-in measures exist to quantify the prediction uncertainty, e.g., the proximity measure for tree ensembles which counts how often a pair of points is located in the same leaf node for all trees of the ensemble \citep{hastie_elemstatlearn}. The same can be done for the pairwise proximity between points in the training and the test set. For instance, given $n$ training observations and a test observation $\boldx$, we can create an ($n \times 1$) vector of proximities which can be used to detect model extrapolations. However, it is desirable to detect model extrapolations via auxiliary extrapolation risk metrics (AERM) which are independent of the trained model.
Detecting an EP is similar in concept to the detection of outliers. Although a unified definition of outliers does not exist, they are generally considered to differ as much from other observations as to suspect they were generated by a different mechanism \citep{hawkins_outliers}. We can therefore consider an outlier to be drawn from a different distribution than the training data (and one that does not overlap with it), which suits our definition of EPs.
In clustering, outliers are often found using local density-based outlier scores such as local outlier probabilities (LOP) \citep{kriegel_lop}. Based on the nearest data points, LOP provides an interpretable score on the scale $[0, 1]$, indicating the probability of a point being an outlier. However, clustering techniques such as LOP are often based on the assumption that the data exhibits a structure of clusters or on assumptions about the clusters' distributions. In theory, one could use various other outlier detection (also referred to as anomaly detection) mechanisms for extrapolation detection, e.g., isolation forests \citep{liu_isolationforest}.
\par
\citet{hooker_CERT} proposes a statistical test to classify a point as an EP or non-EP. It tests whether a point was more likely to be drawn from the data distribution (non-EP) or the uniform distribution (EP). The uniform distribution is used as an uninformative baseline distribution. The AERM $R(\boldx)$ corresponds to:
\begin{equation}
    R(\boldx) = \frac{U(\boldx)}{U(\boldx) + P(\boldx)}
    \label{eq:extrapolation_risk}
\end{equation}
with $U(\boldx)$ being the density function of the uniform distribution and $P(\boldx)$ the density function of the data distribution. $R(\boldx)$ has a range of $[0, 1]$ with 0 indicating the lowest and 1 the highest extrapolation risk. $R(\boldx) > 0.5$  indicates extrapolation. As the support of $U(\boldx)$ we may either choose the recommendations of an application domain expert or the observed feature ranges. Eq. (\ref{eq:extrapolation_risk}) cannot be directly computed, as the density of the training data is unknown. If $\boldx$ falls outside the multivariate envelope of the training data, it is plausible to set $R(\boldx)$ to 1.
\par
We may estimate Eq. (\ref{eq:extrapolation_risk}) by  creating a binary classification problem on a dataset augmented with uniform Monte-Carlo samples \citep{hooker_CERT}. The training data is labeled as the foreground class. Next, artificial data points are sampled from a uniform distribution and labeled as the background class. A predictive model is trained on the augmented dataset and predicts for a given point whether it is more probable that it was drawn from the data distribution or the uniform distribution. We term this approach Monte-Carlo extrapolation classification (MCEC). Fig. \ref{fig:cert} visualizes the concept behind MCEC for a classification tree. Consider two independent standard normally distributed features. We augment the training data with a uniform Monte-Carlo sample with support $[min(x_1), max(x_1)] \times [min(x_2) \times max(x_2)]$ and use CART to partition the feature space into extrapolation areas and non-extrapolation areas. Some training points are located outside of the center rectangles in a low-density end of the bivariate normal distribution. As such, it is correct to be cautious when evaluating predictions in this area, even if a point was drawn from the training data.
\par
\citet{hooker_CERT} argues that in high-dimensional settings, the Monte-Carlo sample will leave lots of areas of the feature space unoccupied which results in poor classification performance. Classification performance may be boosted by directly utilizing distributional information about the uniform distribution instead of a Monte-Carlo sample. This technique termed confidence and extrapolation representation trees (CERT) exploits a property of classification trees which lets one replace the number of Monte-Carlo points per subspace with the expected number of uniform points at each split. Given the feature space $\mathcal{X}$ with $n$ observations and a subspace $\mathcal{X}_{[\;]}$, with $n_{[\;], \text{data}}$ observations, the expected number of uniform points on the subspace $n_{[\;], \;\text{uniform}}$ is proportional to the fraction of feature space hypervolume the subspace occupies:
\begin{equation}
\label{eq:cert_hypervolume}
n_{[\;], \;\text{uniform}} = \frac{\text{hypervolume}(\mathcal{X}_{[\;]})}{\text{hypervolume}(\mathcal{X})} \cdot n_{[\;], \;\text{data}}
\end{equation}
For the tree growing and pruning strategy, CERT uses a mixture of both CART (e.g., splitting based on the Gini index) and C4.5 \citep{quinlan_c4_5} (e.g., missing values and surrogate splits). Apart from letting us directly supply the classification tree with distributional information instead of data, its interpretability is advantageous. The tree partitions the entire feature space at once into hyperrectangles that indicate extrapolation or non-extrapolation areas. \citet{hooker_CERT} argues that CERT provides a markedly lower misclassification rate as opposed to using MCEC with a classification tree. However, it is unclear whether this advantage holds for other classification algorithms.

\begin{figure}
\centering
\includegraphics[width=0.6\textwidth]{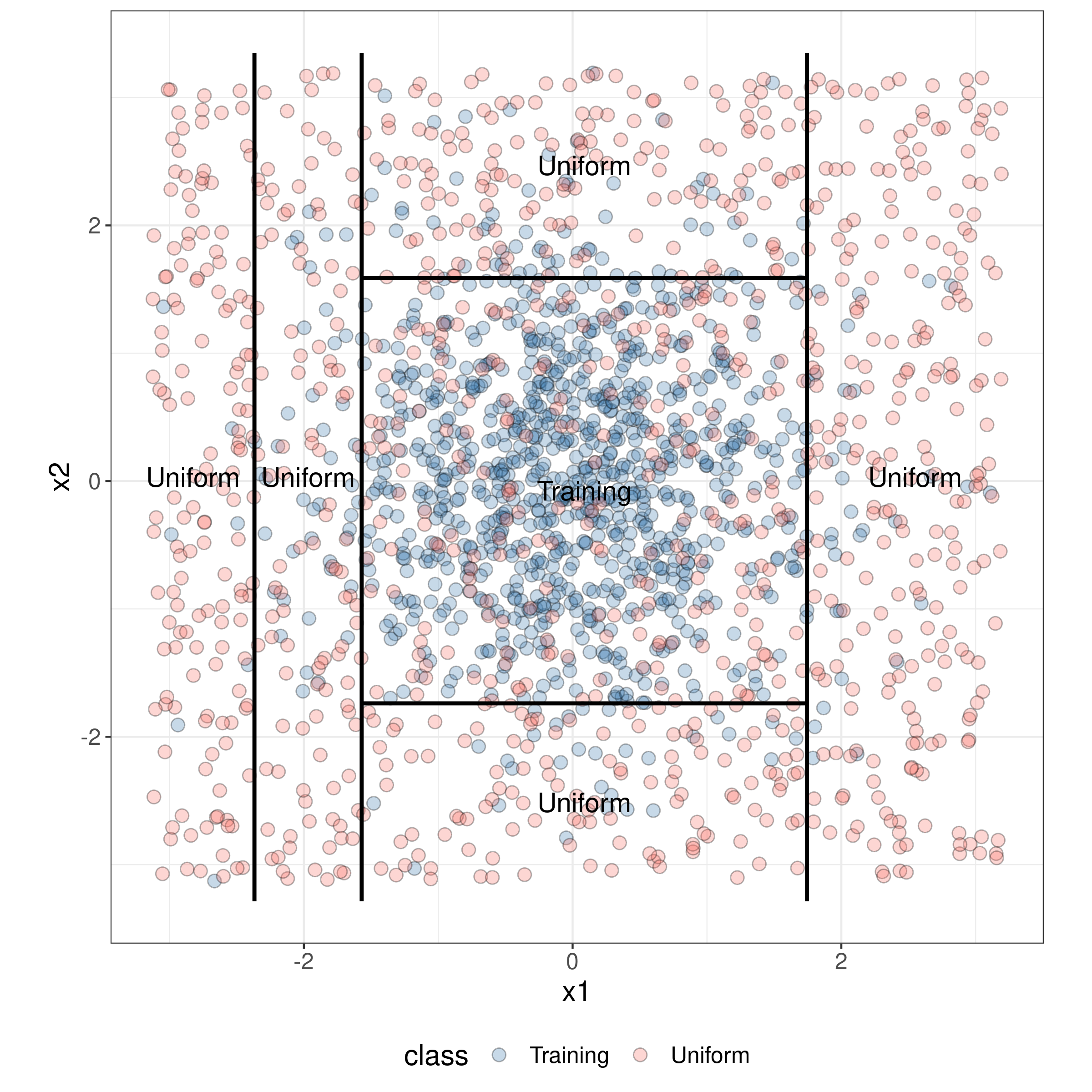}
\caption{\label{fig:cert} \small We augment the training data (blue) with uniform points (red). A classification tree partitions the feature space into non-extrapolation areas (predominantly occupied with training observations) and extrapolation areas (predominantly occupied with uniform Monte-Carlo samples).}
\end{figure}

\section{Proofs}

\subsection{Additive Recovery}
\label{app:additiverecovery}
We provide several proofs on additive recovery based on a prediction function in additive form. Any prediction function can be decomposed into a sum of effect terms of various orders (see Appendix \ref{app:decomposition}).
The sum of effect terms of a feature set $K$ is denoted by $\Theta_{K}(\bm{x}_K)$. For notational simplicity, the union $\{j\} \cup K$ of the $j$-th feature index and the index set $K$ is denoted by $\{j, K\}$. The sum of effect terms is denoted by $\Theta_{\{j, K\}}(x_j, \bm{x}_K)$.

\begin{theorem}[Additive Recovery of Finite Difference]
\label{theorem:additive_recovery}
An FD w.r.t. $x_j$ only recovers terms that depend on $x_j$ and no terms that exclusively depend on $\boldxMinusj$.
\end{theorem}
\begin{proof}
Consider a prediction function $\widehat{f}$ that consists of a sum, including the main effect of $x_j$, denoted by $g_{\{j\}}(x_j)$, a sum of higher order terms (interactions) between $x_j$ and other features $\bm{x}_K$, denoted by $\Theta_{\{j, K\}}(x_j, \bm{x}_K)$, and terms that depend on the remaining features $x_{-\{j, K\}}$, denoted by 
$\Theta_{-\{j, K\}}(\bm{x}_{-\{j, K\}})$:
\begin{equation*}
    \widehat{f}(\boldx) = g_{\{j\}}(x_j) + \Theta_{\{j, K\}}(x_j, \bm{x}_{K}) + \Theta_{-\{j, K\}}(\bm{x}_{-\{j, K\}})
\end{equation*}
It follows that the FD of predictions corresponds to a function that only depends on $x_j$, i.e., it locally recovers the relevant terms on the interval $[x_j + a, x_j + b]$.
\begin{align*}
FD_{j, \boldx, a, b} &= \widehat{f}(x_1, \dots, x_j + a, \dots, x_p) - \widehat{f}(x_1, \dots, x_j + b, \dots, x_p) \\
&= \left[g_{\{j\}}(x_j + a) \; + \; \Theta_{\{j, K\}}(x_j + a, \bm{x}_K) \; + \; \Theta_{-\{j, K\}}(\bm{x}_{-\{j, K\}}) \right]\\
&\phantom{{}={}} - \left[g_{\{j\}}(x_j + b) \; + \; \Theta_{\{j, K\}}(x_j + b, \bm{x}_K) \; + \; \Theta_{-\{j, K\}}(\bm{x}_{-\{j, K\}}) \right]\\
&= g_{\{j\}}(x_j + a) - g_{\{j\}}(x_j + b) + \Theta_{\{j, K\}}(x_j + a, \bm{x}_K) - \Theta_{\{j, K\}}(x_j + b, \bm{x}_K)
\end{align*}
\end{proof}

\begin{corollary}[Additive Recovery of Univariate Forward Marginal Effect]
\label{cor:additive_recovery_marginal_effect}
The univariate fME w.r.t. $x_j$ only recovers terms that depend on $x_j$ and no terms that exclusively depend on $\boldxMinusj$.
\end{corollary}
\begin{proof}
Consider a prediction function $\widehat{f}$ that consists of a sum, including the main effect of $x_j$, denoted by $g_{\{j\}}(x_j)$, a sum of higher order terms (interactions) between $x_j$ and other features $\bm{x}_K$, denoted by $\Theta_{\{j, K\}}(x_j, \bm{x}_K)$, and terms that depend on the remaining features $\bm{x}_{-\{j, K\}}$, denoted by 
$\Theta_{-\{j, K\}}(\bm{x}_{-\{j, K\}})$:
\begin{equation*}
    \widehat{f}(\boldx) = g_{\{j\}}\left(x_j\right) + \Theta_{\{j, K\}}\left(x_j, \bm{x}_{K}\right) + \Theta_{-\{j, K\}}\left(\bm{x}_{-\{j, K\}}\right)
\end{equation*}
The FD w.r.t. $x_j$ is equivalent to the fME w.r.t. $x_j$ with $a = h_j$ and $b = 0$. 
Using Theorem \ref{theorem:additive_recovery}, it follows that:
\begin{align*}
\text{fME}_{\boldx, h_j} &= g_{\{j\}}\left(x_j + h_j\right) - g_{\{j\}}\left(x_j\right) + \Theta_{\{j, \, K\}}\left(x_j + h_j, \bm{x}_{K}\right) - \Theta_{\{j, \, K\}}\left(x_j, \bm{x}_{K}\right)
\end{align*}
\end{proof}

\begin{theorem}[Additive Recovery of Multivariate Forward Marginal Effect]
\label{thm:additive_recovery_multivariate_marginal_effect}
The multivariate fME w.r.t. $\boldxS$ only recovers terms that depend on $\boldxS$ and no terms that exclusively depend on $\boldxMinusS$.
\end{theorem}
\begin{proof}
Consider a feature set $S$. The power set of $S$ excluding the empty set is denoted by $\mathcal{P}^{*} = \mathcal{P}(S) \setminus \; \emptyset$. The prediction function $\widehat{f}$ consists of a sum, including the sum of effects of all subsets of features $K \in \mathcal{P}^{*}$, denoted by $\sum_{K \in \mathcal{P}^{*}}g_{K}\left(\bm{x}_K\right)$, and a sum of terms that depend on the remaining features, denoted by $\Theta_{-S}(\boldxMinusS)$:
\begin{equation*}
    \widehat{f}(\boldx) = \sum_{K \in \mathcal{P}^{*}}g_{K}\left(\bm{x}_K\right) + \Theta_{-S}\left(\boldxMinusS\right)
\end{equation*}
\begin{align*}
\text{fME}_{\boldx, \boldhS} &= \left[\sum_{K \in \mathcal{P}^{*}}g_{K}(\bm{x}_K + \bm{h}_K) + \Theta_{-S}(\boldxMinusS)\right] - \left[\sum_{K \in \mathcal{P}^{*}}g_{K}(\bm{x}_K) + \Theta_{-S}\left(\boldxMinusS\right)\right] \\
&= \sum_{K \in \mathcal{P}^{*}}\left[g_{K}(\bm{x}_K + \bm{h}_K) - g_{K}(\bm{x}_K)\right]
\end{align*}
\end{proof}

\subsection{Relation between Marginal Effects, the Individual Conditional Expectation, and Partial Dependence}
\label{app:proofs_fme_ice_pd}

\begin{theorem}[Equivalence between Forward Marginal Effect and Forward Difference of Individual Conditional Expectation]
\label{theorem:equivalence_fme_ice}
The fME with step size $\boldhS$ is equivalent to the forward difference with step size $\boldhS$ between two locations on the ICE.
\end{theorem}
\begin{proof}
\begin{align*}
\text{fME}_{\boldx, \boldhS} &= \widehat{f}(\boldxS + \boldhS, \boldxMinusS) - \widehat{f}(\boldx) \\
&= \text{ICE}_{\boldx, S}(\boldxS + \boldhS) - \text{ICE}_{\boldx, S}(\boldxS) 
\end{align*}
\end{proof}
\begin{theorem}[Equivalence between Average Marginal Effect and Forward Difference of Partial Dependence for Linear Prediction Functions]
\label{theorem:equivalence_ame_pd}
The AME with step size $\boldhS$ is equivalent to the forward difference with step size $\boldhS$ between two locations on the PD for prediction functions that are linear in $\boldxS$.
\end{theorem}
\begin{proof}
If $\widehat{f}$ is linear in $\boldxS$:
\begin{align}
\label{eq:linear_pd_ame}
\widehat{f}\left(\boldxS^{(i)} + \boldhS\right) = \widehat{f}(\boldxS + \boldhS) \quad \forall \; & \; i \;\in\; \{1, \dots, n\}, \\
& \; \boldxS, \; \boldhS \;\in\; \bigtimes_{j \in S} \mathcal{X}_j \nonumber
\end{align}
It follows:
\begin{align*}
\text{AME}_{\mathcal{D}, \boldhS} &= \frac{1}{n} \sum_{i = 1}^n \left[\widehat{f}\left(\boldxS^{(i)} + \boldhS, \boldxMinusS^{(i)}\right) - \widehat{f}\left(\boldxi\right) \right] \\
&= \frac{1}{n} \sum_{i = 1}^n \widehat{f}\left(\boldxS^{(i)} + \boldhS, \boldxMinusS^{(i)}\right) - \frac{1}{n} \sum_{i = 1}^n \widehat{f}\left(\boldxS^{(i)}, \boldxMinusS^{(i)}\right) \\
&\stackrel{(\ref{eq:linear_pd_ame})}{=}
 \frac{1}{n} \sum_{i = 1}^n \widehat{f}\left(\boldxS + \boldhS, \boldxMinusS^{(i)}\right) - \frac{1}{n} \sum_{i = 1}^n \widehat{f}\left(\boldxS, \boldxMinusS^{(i)}\right) \\
&= PD_{\mathcal{D}, S}\left(\boldxS + \boldhS\right) - PD_{\mathcal{D}, S}\left(\boldxS\right)
\end{align*}
\end{proof}

\normalsize
\bibliographystyle{spbasic}
\bibliography{bibfile}   

\end{document}